\documentclass[11pt]{amsart}
\def\isdraft{0}

\usepackage{mathtools}
\usepackage{amsmath,amsfonts,amsthm,amssymb}
\usepackage[dvipsnames]{xcolor}
\usepackage{bbm}
\usepackage{graphicx}
\usepackage{amsaddr} % remove in amsbook, JAIR, IOS-Press
\usepackage{prettyref} % achtung: position dieses kommandos scheint essentiell zu sein
\usepackage[utf8]{inputenc}
\usepackage{enumitem}
\usepackage[hyphens]{url} % remove in JAIR
\usepackage{tikz-cd}
\usepackage{tikz}
\usetikzlibrary{positioning}
\usetikzlibrary{arrows}
\usepackage{bussproofs}
\usepackage[mathscr]{euscript}
\usepackage{hyperref}
\usepackage{a4wide}
\usepackage[color=black,textcolor=white\if\isdraft0,disable\fi]{todonotes}
\usepackage{stackengine}
\usepackage{stmaryrd}
\usepackage{wrapfig}
\usepackage{marginnote}
\usepackage{float}
\graphicspath{{Figures/}}
\usepackage{footnote}
\usepackage{tabularx}
\usepackage[normalem]{ulem}
\usepackage{pifont}
\usepackage[most]{tcolorbox}
\usepackage{shuffle}
\usepackage{tikz}
\usetikzlibrary{automata,positioning}

\newtheorem{theorem}{Theorem}

\newtheorem{claim}[theorem]{Claim}

\newtheorem{corollary}[theorem]{Corollary}
\newtheorem{fact}[theorem]{Fact}
\newtheorem{lemma}[theorem]{Lemma}
\newtheorem{proposition}[theorem]{Proposition}

\theoremstyle{definition} % TPLP remove

\newtheorem{conclusion}[theorem]{Conclusion}

\newtheorem{definition}[theorem]{Definition}

\newtheorem{example}[theorem]{Example}

\newtheorem{notation}[theorem]{Notation}

\newtheorem{problem}[theorem]{Problem}
\newtheorem{pseudocode}[theorem]{Pseudocode}

\newtheorem{remark}[theorem]{Remark}
\newtheorem{Todo}[theorem]{Todo}

\newrefformat{§}{§\ref{#1}}
\newrefformat{a}{Answer \ref{#1}}
\newrefformat{alg}{Algorithm \ref{#1}}
\newrefformat{ax}{Appendix \ref{#1}}
\newrefformat{cl}{Claim \ref{#1}}
\newrefformat{con}{Conclusion \ref{#1}}
\newrefformat{conv}{Convention \ref{#1}}
\newrefformat{cj}{Conjecture \ref{#1}}
\newrefformat{c}{Corollary \ref{#1}}
\newrefformat{q}{(\ref{#1})}
\newrefformat{e}{Example \ref{#1}}
\newrefformat{exe}{Exercise \ref{#1}}
\newrefformat{d}{Definition \ref{#1}}
\newrefformat{Done}{Done \ref{#1}}
\newrefformat{f}{Fact \ref{#1}}
\newrefformat{fig}{Figure \ref{#1}}
\newrefformat{h}{Hypothesis \ref{#1}}
\newrefformat{idea}{Idea \ref{#1}}
\newrefformat{i}{Item \ref{#1}}
\newrefformat{l}{Lemma \ref{#1}}
\newrefformat{N}{Note \ref{#1}}
\newrefformat{n}{Notation \ref{#1}}
\newrefformat{o}{Observation \ref{#1}}
\newrefformat{pr}{Problem \ref{#1}}
\newrefformat{p}{Proposition \ref{#1}}
\newrefformat{pc}{Pseudocode \ref{#1}}
\newrefformat{Q}{Q. \ref{#1}}
\newrefformat{r}{Remark \ref{#1}}
\newrefformat{tab}{Table \ref{#1}}
\newrefformat{t}{Theorem \ref{#1}}
\newrefformat{T}{TODO \ref{#1}}
\newrefformat{traum}{Traum \ref{#1}}
\newrefformat{w}{Warning \ref{#1}}

\newcommand{\boxblacktriangle}{\mathrel{\ooalign{$\square$\cr\kern0.07ex\hbox{\scalebox{0.9}{$\blacktriangle$}}}}}
\newcommand{\boxtriangle}{\mathrel{\ooalign{$\square$\cr\kern0.07ex\hbox{\scalebox{0.9}{$\triangle$}}}}}

\newcommand{\righttherefore}{:\joinrel\cdot\,}

\newlist{todolist}{itemize}{2}
\setlist[todolist]{label=$\square$}

\usepackage[top=0.7cm,bottom=0.5cm,left=1cm,right=1cm]{geometry} % Seitenraender

\if\isdraft1
\fi

\title{
	Algebraic anti-unification
}
\author{
	Christian Anti\'c
}
\address{
	christian.antic@icloud.com\\
	Vienna University of Technology\\
	Vienna, Austria
}

\begin{document}

\begin{abstract}
	Abstraction is key to human and artificial intelligence as it allows one to identify common structure in otherwise distinct objects or situations. Anti-unification (or generalization) is the branch of theoretical computer science and artificial intelligence that studies abstraction and has found applications in areas such as inductive logic programming, program synthesis, and analogy-making. To date, anti-unification has been studied almost exclusively from a syntactic perspective. In this paper, we initiate an algebraic (i.e.\ semantic) theory of anti-unification in the general setting of universal algebra, thereby extending anti-unification from term-based representations to arbitrary algebras and beyond equational theories. In particular, we introduce the notions of algebraic generalization ordering and minimally general generalization, establish basic structural properties, prove compatibility with homomorphisms and isomorphisms, and investigate computability in finite unary algebras and finite algebras via automata-theoretic methods.
\end{abstract}

\maketitle

\section{Introduction}\label{§:I}

Abstraction is key to human and artificial intelligence (AI) as it allows one to see common structure in otherwise distinct objects or situations \cite{Giunchiglia92,Saitta13} and as such it is a key element for generality in computer science and artificial intelligence \cite{McCarthy87,Kramer07}. It has been studied in the fields of theorem proving (e.g.\ \cite{Plaisted81}) and knowledge representation and reasoning (KR\&R) by a number of authors (e.g.\ \cite{Knoblock94,Sacerdoti74}), and it has recently gained momentum in answer set programming \cite{Saribatur20,Saribatur21}, one of the most prominent formalisms in the field of KR\&R (see e.g.\ \cite{Brewka11,Lifschitz19}).

Anti-unification (or generalization) (cf.\ \cite{Cerna23}) is \textit{the} field of mathematical logic and theoretical computer science studying abstraction. It is the ``dual'' operation to the well-studied unification operation (cf.\ \cite{Baader01}). More formally, given two terms $s$ and $t$, unification searches for a substitution --- the \textit{most general unifier} --- $\sigma$ satisfying $s\sigma=t\sigma$, whereas syntactic anti-unification searches for the \textit{least general generalization} $u$ such that $s=u\sigma$ and $t=u\theta$, for some substitutions $\sigma,\theta$. Notice the difference between the two operations: while unification computes a substitution, anti-unification computes a term.

Syntactic anti-unification has been introduced by \cite{Plotkin70} and \cite{Reynolds70}, and it has found numerous applications in theoretical computer science and artificial intelligence, as for example in inductive logic programming \cite{Muggleton91} (cf.\ \cite{Cropper21,Cropper22}), programming by example \cite{Gulwani16}, library learning and compression \cite{Cao23}, and, in the form of $E$-generalization (i.e.\ anti-unification modulo theory) \cite{Heinz95,Burghardt05} in analogy-making \cite{Weller07,Schmidt14} (for further applications, see e.g.\ \cite{Barwell18,deSousa21,Vanhoof19}).

Despite a substantial body of work on anti-unification modulo equational theories such as associativity, commutativity, identity, and idempotency (see e.g.\ \cite{Heinz95,Burghardt05,Cerna20,Cerna23}), anti-unification has largely been studied within a syntactic framework based on terms and equational reasoning. The purpose of this paper is to develop a genuinely algebraic---that is, semantic---theory of anti-unification in the general setting of universal algebra. We note that algebraic anti-unification naturally generalizes $E$-anti-unification: whenever the underlying algebra admits an equational axiomatization $E$, algebraic and $E$-anti-unification are two sides of the same coin. However, the algebraic perspective applies more broadly to arbitrary algebras, including those without an equational axiomatization. More importantly, in practice, we often deal with algebras representing objects such as graphs, automata, or transition systems, and wish to reason about these objects directly rather than first searching for a complete equational axiomatization. For example, in the multiplicative arithmetical setting, it is the unique prime factorization property of the natural numbers that matters. The algebraic approach allows us to reason directly about such semantic properties of structures, independently of any specific equational axiomatization.

The contributions of this paper are foundational in nature. We introduce the central notions of algebraic generalization ordering and minimally general generalization (\prettyref{§:AU}) and establish basic structural properties including a classification of generalization types. A key structural result is that algebraic anti-unification is compatible with structure-preserving mappings: we prove a Homomorphism Lemma and an Isomorphism Theorem showing that minimally general generalizations are preserved under isomorphisms (\prettyref{§:H}). On the computational side, we show that in the practically important class of finite unary algebras --- equivalently, semiautomata --- minimally general generalizations are computable via standard techniques from the theory of finite automata (\prettyref{§:FUA}), and that tree automata provide the appropriate framework for general finite algebras (\prettyref{§:FA}). We further introduce a fragment-based approach to tractable computation via $(k,\ell)$-fragments (\prettyref{§:Fragments}).

Several fundamental questions are raised but left open, and we regard this as a feature rather than a limitation: the paper is intended to open a research programme, not to close one. The central open problem --- classifying the generalization type of a given pair of algebras --- defines a rich agenda at the intersection of universal algebra, model theory, and computational complexity. The extension to bilingual settings, the study of finitely representable infinite structures \cite{Blumensath00,Blumensath04}, and the development of efficient algorithms for practical use each constitute substantial lines of future work that this paper is intended to motivate and enable.

The initial motivation for developing an algebraic theory of anti-unification comes from two recent applications in AI: an abstract algebraic notion of similarity \cite{Antic23-2} and a framework of analogical proportions \cite{Antic22,Antic23-22} in the general setting of universal algebra, which has been further developed in monounary algebras \cite{Antic22-2}, two-element boolean algebras \cite{Antic21-3}, and applied to automatic logic programming by analogy \cite{Antic23-23}. In both cases, sets of generalizations arise naturally from the definitions. The role of minimally general generalizations is already well understood in the special case of free term algebras \cite[§Tree proportions]{Antic23-22}; the present paper provides the general algebraic framework needed to extend this understanding to arbitrary algebras.

The present paper is concerned with the mathematical foundations of algebraic anti-unification. While the long-term motivation stems from applications in AI, our immediate objective is to establish the underlying algebraic theory before addressing algorithmic and application-oriented questions.

The paper is structured as follows. \prettyref{§:AU} introduces algebraic anti-unification and derives elementary properties. \prettyref{§:GT} defines the generalization-type of a pair of algebras. \prettyref{§:H} establishes the Homomorphism Lemma and Isomorphism Theorem. \prettyref{§:Monounary} works out the theory for monounary algebras. \prettyref{§:FUA} treats finite unary algebras via automata-theoretic methods. \prettyref{§:Fragments} introduces $(k,\ell)$-fragments which are used in \prettyref{§:FA} treating algebraic anti-unification in general finite algebras via tree automata.

\section{Preliminaries}

We expect the reader to be fluent in basic general algebra as it is presented for example in \cite[§II]{Burris00}.

A \textit{\textbf{language}} $L$ of algebras is a set of \textit{\textbf{function symbols}}\footnote{We omit constant symbols since constants are identified with 0-ary functions.} together with a \textit{\textbf{rank function}} $r:L\to\mathbb N$. Moreover, we assume a denumerable set $X$ of \textit{\textbf{variables}} distinct from $L$. Terms are formed as usual from variables in $X$ and function symbols in $L$ and we denote the set of all such terms by $T_{L,X}$. The set of variables occurring in a term $s$ are denoted by $X(s)$ and $s$ has \textit{\textbf{rank}} $k$ iff $X(s)=\{x_1,\ldots,x_k\}$. The rank of $s$ is denoted by $r(s)$.

An \textit{\textbf{$L$-algebra}} $\mathfrak A$ consists of a non-empty set $A$, the \textit{\textbf{universe}} of $\mathfrak A$, and for each function symbol $f\in L$, a function $f^\mathfrak A:A^{r(f)}\to A$, the \textit{\textbf{functions}} of $\mathfrak A$ (the \textit{\textbf{distinguished elements}} of $\mathfrak A$ are the 0-ary functions). We will not distinguish between distinguished elements and their 0-ary function symbols. An algebra is \textit{\textbf{injective}} iff each of its function is injective. Notice that every distinguished element $a$ of $\mathfrak A$ is a 0-ary function $a:A^0\to A$ which maps some single dummy element in $A^0$ to $a$ and thus \textit{is} injective.

Every term $s$ induces a \textit{\textbf{term function}} $s^\mathfrak A$ on $\mathfrak A$ in the usual way. In this paper, we shall not distinguish between terms inducing the same function on the underlying algebra, which is common practice in mathematics where one does usually not distinguish, for example, between the terms $2x^2$ and $(1+1)x^2$ in arithmetic.

\begin{fact}\label{f:injective} Every term function induced by injective functions is injective.
\end{fact}

The \textit{\textbf{term algebra}} $\mathfrak T_{L,X}$ over $L$ and $X$ is the algebra we obtain by interpreting each function symbol $f\in L$ by
\begin{align*} 
	f^{\mathfrak T_{L,X}} : T_{L,X}^{r(f)}\to T_{L,X} : (s_1, \ldots, s_{r(f)})\mapsto f(s_1,\ldots,s_{r(f)}).
\end{align*} 

A \textit{\textbf{substitution}} is a mapping $\sigma:X\to T_{L,X}$, where $\sigma(x) := x$ for all but finitely many variables, extended from $X$ to terms in $T_{L,X}$ inductively as usual, and we write $s\sigma$ for the application of $\sigma$ to the term $s$. 

For two terms $s,t\in T_{L,X}$, we define the \textit{\textbf{syntactic generalization ordering}} by
\begin{align*} 
	s\lesssim t \quad:\Leftrightarrow\quad s=t\sigma,\quad\text{for some substitution $\sigma$}.
\end{align*}

A \textit{\textbf{homomorphism}} is a mapping $H: \mathfrak{A\to B}$ such that for any function symbol $f\in L$ and elements $a_1,\ldots,a_{r(f)}\in A$,
\begin{align*} 
	H(f^ \mathfrak A(a_1, \ldots, a_{r(f)})) = f^ \mathfrak B(H(a_1), \ldots, H(a_{r(f)})).
\end{align*} An \textit{\textbf{isomorphism}} is a bijective homomorphism. 
% The application $Hs$ of a homomorphism $H$ to an $L$-term $s$ is defined point-wise by
% \begin{align*} 
% 	(Hs)^ \mathfrak A \textbf{o}:=Hs^ \mathfrak A(\textbf{o})=s^ \mathfrak B(H(\textbf{o})),\quad\text{for all $\textbf{o}\in A^{r(s)}$},
% \end{align*} where $H(\textbf{o})$ means component-wise application of $H$ to $\textbf{o}$.

\section{Anti-unification}\label{§:AU}

In this section, we introduce algebraic anti-unification and derive some elementary observations.

In the rest of the paper, let $\mathfrak A$ and $\mathfrak B$ be $L$-algebras over some joint language of algebras $L$, and let $\mathfrak{(A,B)}$ be a pair of $L$-algebras. We will always write $\mathfrak A$ instead of $\mathfrak{(A,A)}$.

\begin{definition}\label{d:downarrow} Given an $L$-term $s$, define
\begin{align*} 
	\downarrow_\mathfrak A s:= \left\{s^\mathfrak A(\textbf{o})\in A \;\middle|\; \textbf{o}\in A^{r(s)}\right\}.
\end{align*} In case $a\in\ \downarrow_\mathfrak A s$, we say that $s$ is a \textit{\textbf{generalization}} of $a$ and $a$ is an \textit{\textbf{instance}} of $s$.
\end{definition}

\begin{fact} Every distinguished element $a\in A$ is an instance of itself\footnote{More precisely, of the constant symbol denoting it.} since
\begin{align}\label{eq:downarrow_a} 
	a\in\ \downarrow_\mathfrak A a=\{a\}.
\end{align}
\end{fact}

\begin{definition} Define the \textit{\textbf{(semantic) generalization ordering}} for two $L$-terms $s$ and $t$ in $\mathfrak{(A,B)}$ by
\begin{align*} 
	s\sqsubseteq_{ \mathfrak{(A,B)}} t \quad:\Leftrightarrow\quad \downarrow_\mathfrak A s\subseteq\ \downarrow_\mathfrak A t \quad\text{and}\quad \downarrow_\mathfrak B s\subseteq\ \downarrow_\mathfrak B t,
\end{align*} and
\begin{align*} 
	s\equiv_{ \mathfrak{(A,B)}} t \quad:\Leftrightarrow\quad s\sqsubseteq_{ \mathfrak{(A,B)}} t \quad\text{and}\quad t\sqsubseteq_{ \mathfrak{(A,B)}} s.
\end{align*}
\end{definition}

\begin{example} $0x\equiv_{(\mathbb N,\cdot,0)} 0$.
\end{example}

\begin{proposition} The generalization ordering $\sqsubseteq_{ \mathfrak{(A,B)}}$ is a pre-order between $L$-terms, for any pair of $L$-algebras $\mathfrak{(A,B)}$, that is, it is reflexive and transitive.
\end{proposition}
% \begin{proof} 
% \todo[inline]{}
% \end{proof}

\begin{fact} For any distinguished element $a\in A$, we have by \prettyref{eq:downarrow_a}:
\begin{align*} 
	a\sqsubseteq_\mathfrak A s \quad\Leftrightarrow\quad a\in\ \downarrow_\mathfrak A s
\end{align*} and
\begin{align*} 
	s\sqsubseteq_\mathfrak A a \quad\Leftrightarrow\quad s\equiv_\mathfrak A a \quad\Leftrightarrow\quad \text{$s^\mathfrak A$ is a constant function with value $a$.}
\end{align*} For two distinguished elements $a,b\in A$, we have
\begin{align*} 
	a\sqsubseteq_\mathfrak A b \quad\Leftrightarrow\quad \downarrow_\mathfrak A a\subseteq\ \downarrow_\mathfrak A b \quad\Leftrightarrow\quad \{a\}\subseteq\{b\} \quad\Leftrightarrow\quad a=b \quad\Leftrightarrow\quad a\equiv_ \mathfrak A b.
\end{align*} 
\end{fact}

\begin{remark} Notice that our semantic generalization ordering $\sqsubseteq$ differs from the usual one in \textit{syntactic} anti-unification $\lesssim$ where a term $s$ is said to be \textit{\textbf{more general than}} a term $t$ iff there is a substitution $\sigma$ such that $t=s\sigma$.
\todo[inline]{provide examples}
\end{remark}

\begin{definition} For any element $a\in A$, we define
\begin{align*} 
	\uparrow_ \mathfrak A a:=\{s\in T_{L,X} \mid a\in\ \downarrow_ \mathfrak A s\},
\end{align*} extended to elements $a\in A$ and $b\in B$ by
\begin{align*} 
	a\uparrow_{ \mathfrak{(A,B)}} b:=(\uparrow_ \mathfrak A a)\cap (\uparrow_ \mathfrak B b).
\end{align*}
\end{definition}

We shall now introduce the main notion of the paper: 

\begin{definition} Define the set of \textit{\textbf{minimally general generalizations}} (or \textit{\textbf{mggs}}) of two elements $a\in A$ and $b\in B$ in $\mathfrak{(A,B)}$ by
\begin{align*} 
	a\Uparrow_{ \mathfrak{(A,B)}} b:=\min_{\sqsubseteq_{ \mathfrak{(A,B)}}}(a\uparrow_{ \mathfrak{(A,B)}} b).
\end{align*} In case $a\Uparrow_{ \mathfrak{(A,B)}} b=\{s\}$ contains a single generalization, $s$ is called the \textit{\textbf{least general generalization}} (or \textit{\textbf{lgg}}) of $a$ and $b$ in $\mathfrak{(A,B)}$.
\end{definition}
 
Here $\min_{\sqsubseteq_{ \mathfrak{(A,B)}}}$ is the function computing the \textit{set} of minimal generalizations with respect to $\sqsubseteq_{ \mathfrak{(A,B)}}$ (not a single minimal generalization), which means that in case there are \textit{no} minimal generalizations, the computed set is empty. Notice that in case $a\in A$ is a distinguished element, we always have\footnote{More precisely, the constant symbol for the element $a$ is in $a\Uparrow_ \mathfrak A a$.}
\begin{align*} 
	a\in a\Uparrow_\mathfrak A a.
\end{align*}

\begin{example}\label{e:BOOL} Let $\mathfrak{BOOL}:=(\{0,1\},\lor,\neg,\{0,1\})$ be the 2-element boolean algebra with disjunction and negation (and therefore with all boolean functions) where both truth values are distinguished elements. Notice that terms in $T_{\{\lor,\neg,0,1\}, X}$ are propositional formulas with variables from $X$ in the usual sense. We have
\begin{align*} 
	\uparrow_ \mathfrak{BOOL} 0&=\{s\in T_{\{\lor,\neg,0,1\}, X} \mid \text{$s$ is falsifiable}\},\\
	\uparrow_ \mathfrak{BOOL} 1&=\{s\in T_{\{\lor,\neg,0,1\}, X} \mid \text{$s$ is satisfiable}\},\\
	0\uparrow_ \mathfrak{BOOL} 1&=\{s\in T_{\{\lor,\neg,0,1\}, X} \mid \text{$s$ is satisfiable and falsifiable}\}.
\end{align*} Since all satisfiable and falsifiable formulas induce the same semantic behavior in the two-element Boolean algebra, we have
\begin{align*} 
	0\Uparrow_ \mathfrak{BOOL} 1=0\uparrow_ \mathfrak{BOOL} 1.
\end{align*}
\end{example}

\section{Generalization type}\label{§:GT}

The following definition is an adaptation of the nomenclature in the excellent survey on anti-unification by \cite[Definition 5]{Cerna23} (triviality is new):

\begin{definition} The \textit{\textbf{generalization type}} of a pair of $L$-algebras $\mathfrak{(A,B)}$ is called
\begin{itemize}
	\item \textit{\textbf{nullary}} iff $a\Uparrow_{ \mathfrak{(A,B)}} b=\emptyset$, for all $a\in A$ and $b\in B$;
	\item \textit{\textbf{unitary}} iff $|a\Uparrow_{ \mathfrak{(A,B)}} b|=1$, for all $a\in A$ and $b\in B$;
	\item \textit{\textbf{finitary}} iff
		\begin{itemize}
			\item $|a\Uparrow_{ \mathfrak{(A,B)}} b|<\infty$, for all $a\in A$ and $b\in B$;
			\item $|a\Uparrow_{ \mathfrak{(A,B)}} b|>1$, for some $a\in A$ and $b\in B$;
		\end{itemize}
	\item \textit{\textbf{infinitary}} iff $|a\Uparrow_{ \mathfrak{(A,B)}} b|=\infty$, for some $a\in A$ and $b\in B$;
	\item \textit{\textbf{trivial}} iff $a\Uparrow_{ \mathfrak{(A,B)}} b=T_{L,X}$, for all $a\in A$ and $b\in B$.
\end{itemize} We say that $\mathfrak A$ has some property from above iff the pair $\mathfrak{(A,A)}$ has that property.
\end{definition}

% The main problems to be studied in this context are:
% \begin{description}
% 	\item[Generalization type] What is the generalization type of a pair of $L$-algebras $\mathfrak{(A,B)}$?

% 	\item[Generalization algorithm] How to compute (or enumerate) $a\Uparrow_{ \mathfrak{(A,B)}} b$ for given $a\in A$, $b\in B$, and $\mathfrak{(A,B)}$.
% \end{description}

\begin{fact}\label{f:uparrow_a=x} $\uparrow_ \mathfrak A a=\{x\}$ implies $a\Uparrow_{ \mathfrak{(A,B)}} b=\{x\}$, for all $b\in B$. Hence, algebras containing elements having only the trivial generalization $x$ cannot be nullary. 
\end{fact}

\begin{example} In the monounary algebra $\mathfrak A=(\{a,b\},S)$ given by
\begin{center}
\begin{tikzpicture} 
	\node (a) {$a$};
	\node (b) [right=of a] {$b$};
	\draw[->] (a) to [edge label={$S$}] (b);
	\draw[->] (b) to [edge label'={$S$}][loop] (b);
\end{tikzpicture}
\end{center} we have
\begin{align*} 
	\uparrow_\mathfrak A a=\{x\}
\end{align*} and thus by \prettyref{f:uparrow_a=x} we have
\begin{align*} 
	a\Uparrow_ \mathfrak A b=\{x\}.
\end{align*} This shows that $\mathfrak A$ is unitary.
\end{example}

\begin{example} The generalization type of $(\mathbb N,S)$, where $S$ is the successor function, is unitary by the forthcoming \prettyref{c:(N,S)_unitary}.
\end{example}

\begin{fact} $a\uparrow_{ \mathfrak{(A,B)}} b=\{x\}$ implies $a\Uparrow_{ \mathfrak{(A,B)}} b=a\uparrow_{ \mathfrak{(A,B)}} b\neq\emptyset$.
\end{fact}

\begin{proposition} In any pair of algebras $\mathfrak{(A,B)}$ with $|A|=|B|=1$, we have $a\Uparrow_{ \mathfrak{(A,B)}} b=a\uparrow_{ \mathfrak{(A,B)}} b\neq\emptyset$. Hence, pairs of algebras consisting of a single element each cannot be nullary.
\end{proposition}
\begin{proof} Every generalization $s\in a\uparrow_{ \mathfrak{(A,B)}} b$ has to satisfy $\downarrow_\mathfrak A s=\{a\}$ and $\downarrow_\mathfrak B s=\{b\}$, which means that there cannot be a generalization $t$ such that $\downarrow_{ \mathfrak{(A,B)}} t\subsetneq\ \downarrow_{ \mathfrak{(A,B)}} s$.
\end{proof}

\begin{proposition} The algebra $(A)$ consisting only of its universe is unitary and trivial.
\end{proposition}
\begin{proof} A direct consequence of $a\Uparrow_{(A)} b=\{x\}$, for all $a,b\in A$ (notice that $a\Uparrow_{(A)} a\neq\{a\}$ since $a$ is not a distinguished element of $(A)$ and thus cannot be used to form terms).
\end{proof}

\begin{fact} $s\in a\Uparrow_{ \mathfrak{(A,B)}} b$ for any term $s$ satisfying $\downarrow_ \mathfrak A s=\{a\}$ and $\downarrow_ \mathfrak B s=\{b\}$.
\end{fact}

Notice that in term algebras, syntactic and semantic generalization orderings (and thus anti-unification) coincide:

\begin{lemma}\label{l:s_lesssim_t} For any $L$-terms $s,t\in T_{L,X}$, we have
\begin{align*} 
	s\lesssim t \quad\Leftrightarrow\quad s\sqsubseteq_{ \mathfrak T_{L,X}} t.
\end{align*}
\end{lemma}
\begin{proof} We have the following equivalences:
\begin{align*} 
	s\lesssim t
		&\quad\Leftrightarrow\quad s=t(\textbf{o}),\quad\text{for some $\textbf{o}\in T_{L,X}^{r(t)}$}\\
		& \quad\Leftrightarrow\quad \downarrow_{ \mathfrak T_{L,X}} s\subseteq\ \downarrow_{ \mathfrak T_{L,X}} t\\
		& \quad\Leftrightarrow\quad s\sqsubseteq_{ \mathfrak T_{L,X}} t.
\end{align*}
\end{proof}

\begin{theorem} Term algebras are unitary.
\end{theorem}
\begin{proof} A direct consequence of \prettyref{l:s_lesssim_t} and the well-known fact that term algebras are unitary by the classical results of \cite{Plotkin70} and \cite{Reynolds70}, and their original algorithms (plus the simple algorithm in \cite{Huet76}) can be used to compute the least general generalization of two terms.
\end{proof}

\begin{theorem} In every injective algebra $\mathfrak A$, we have $a\Uparrow_ \mathfrak A b\neq\emptyset$, for all $a,b\in A$. Consequently, injective algebras cannot be nullary.
\end{theorem}
\begin{proof} We have $a\Uparrow_ \mathfrak A b=\emptyset$ iff for each $s\in a\uparrow_ \mathfrak A b$ there is some $t\in a\uparrow_ \mathfrak A b$ such that $t\sqsubset s$, which is equivalent to $\downarrow_ \mathfrak A t\subset\ \downarrow_ \mathfrak A s$. Since $t^ \mathfrak A$ is injective by assumption (\prettyref{f:injective}), we must have $t^ \mathfrak A(\textbf{o})\neq t^ \mathfrak A\textbf{u}$ for every $\mathbf{o,u}\in A^{r(t)}$, which means that
\begin{align*} 
	|\downarrow_ \mathfrak A t\ |=\left|\left\{t^ \mathfrak A(\textbf{o}) \;\middle|\; \textbf{o}\in A^{r(t)}\right\}\right|=|A^{r(t)}|
\end{align*} and, analogously,
\begin{align*} 
	|\downarrow_ \mathfrak A s\ |=|A^{r(s)}|.
\end{align*} Since $\downarrow_ \mathfrak A t\subsetneq\ \downarrow_ \mathfrak A s$, we have
\begin{align*} 
	|\downarrow_ \mathfrak A t\ |<|\downarrow_ \mathfrak A s\ |,
\end{align*} thus
\begin{align*} 
	|A^{r(t)}|<|A^{r(s)}|,
\end{align*} and hence
\begin{align*}
	r(t)<r(s).
\end{align*} Since the rank of every function is finite, we have $rs<\infty$ and thus there can be only finitely many $t$'s with $r(t)<r(s)$. Hence, there must be some $t'\in a\uparrow b$ for which there can be no $t''\in a\uparrow b$ with $t''\sqsubset t'$ --- we thus conclude $t'\in a\Uparrow_ \mathfrak A b$ and $a\Uparrow_ \mathfrak A b\neq\emptyset$.
\end{proof}

% \section{Structure-preserving mappings}
\section{Homomorphisms}\label{§:H}

In this section, we show that algebraic anti-unification is compatible with structure-preserving mappings.

\begin{lemma}[Homomorphism Lemma]\label{l:HL} For any homomorphism $H: \mathfrak{A\to B}$, any $L$-term $s$, and any elements $a,b\in A$,
\begin{align} 
	H(\downarrow_ \mathfrak A s)&\subseteq\ \downarrow_ \mathfrak B s,\\
	\label{eq:Ha_uparrow_H(b)} a\uparrow_ \mathfrak A b&\subseteq H(a)\uparrow_ \mathfrak B H(b).
\end{align} In case $H$ is an isomorphism, we have
\begin{align} 
	\label{eq:H_downarrow_A_s=downarrow_B_s} H(\downarrow_ \mathfrak A s)&=\ \downarrow_ \mathfrak B s,\\
	a\uparrow_ \mathfrak A b &= H(a)\uparrow_ \mathfrak B H(b).
\end{align}
\end{lemma}
\begin{proof} Since $H$ is a homomorphism by assumption, we have
\begin{align*} 
	H(\downarrow_ \mathfrak A s)=\left\{H(s^ \mathfrak A(\textbf{o})) \;\middle|\; \textbf{o}\in A^{r(s)} \right\}= \left\{s^ \mathfrak B(H(\textbf{o})) \;\middle|\; \textbf{o}\in A^{r(s)}\right\}\subseteq \left\{s^ \mathfrak B(\textbf{o}) \;\middle|\; \textbf{o}\in B^{r(s)}\right\}=\ \downarrow_ \mathfrak B s.
\end{align*} In case $H$ is an isomorphism, we have ``$=$'' instead of ``$\subseteq$'' in the above computation.

We shall now prove \prettyref{eq:Ha_uparrow_H(b)} by showing that every term $s\in a\uparrow_ \mathfrak A b$ is in $H(a)\uparrow_ \mathfrak A H(b)$. Since $s\in a\uparrow_ \mathfrak A b$ holds by assumption, we have
\begin{align*} 
	a=s^ \mathfrak A(\textbf{o}) \quad\text{and}\quad b=s^ \mathfrak A(\textbf{u}),
\end{align*} for some $\mathbf{o,u}\in A^{r(s)}$. Since $H$ is a homomorphism, we thus have
\begin{align*} 
	H(a)=H(s^ \mathfrak A(\textbf{o}))=s^ \mathfrak B(H(\textbf{o})) \quad\text{and}\quad H(b)=H(s^ \mathfrak A(\textbf{u}))=s^ \mathfrak B(H(\textbf{u})),
\end{align*} which shows \prettyref{eq:Ha_uparrow_H(b)}.

It remains to show that in case $H$ is an isomorphism, we have
\begin{align*} 
	H(a)\uparrow_ \mathfrak B H(b)\subseteq a\uparrow_ \mathfrak A b.
\end{align*} For this, let $$s\in H(a)\uparrow_ \mathfrak B H(b)$$ be a generalization of $H(a)$ and $H(b)$ in $\mathfrak B$ which by definition means that there are some $\mathbf{o,u}\in B^{r(s)}$ such that
\begin{align*} 
	H(a)&=s^ \mathfrak B(\textbf{o}),\\
	H(b)&=s^ \mathfrak B(\textbf{u}).
\end{align*} Since $H$ is an isomorphism, its inverse $H^{-1}$ is an isomorphism as well, which yields
\begin{align*} 
	a&=H^{-1}(s^ \mathfrak B(\textbf{o}))=s^ \mathfrak A(H^{-1}(\textbf{o})),\\
	b&=H^{-1}(s^ \mathfrak B(\textbf{u}))=s^ \mathfrak A(H^{-1}(\textbf{u})).
\end{align*} This shows
\begin{align*} 
	s\in a\uparrow_ \mathfrak A b.
\end{align*}
\end{proof}

\begin{theorem}[Isomorphism Theorem]\label{t:HT} For any isomorphism $H: \mathfrak{A\to B}$ and elements $a,b\in A$,
\begin{align}
	a\Uparrow_ \mathfrak A b=H(a)\Uparrow_ \mathfrak B H(b).
\end{align}
\end{theorem}
\begin{proof} $(\subseteq)$ We show that each
\begin{align}\label{eq:s_in_a_Uparrow_A_b}
	s\in a\Uparrow_ \mathfrak A b
\end{align} is contained in $H(a)\Uparrow_ \mathfrak B H(b)$ for the isomorphism $H$. By \prettyref{l:HL}, we have $a\uparrow_ \mathfrak A b\subseteq H(a)\uparrow_ \mathfrak B H(b)$. It thus remains to show that $s$ is $\sqsubseteq_ \mathfrak B$-minimal. 

Suppose there is some $t\in H(a)\uparrow_ \mathfrak B H(b)$ such that $t\sqsubset_ \mathfrak B s$, that is,
\begin{align}\label{eq:t_subsetneq_s} 
	\downarrow_ \mathfrak B t\subsetneq\ \downarrow_ \mathfrak B s.
\end{align} Since $t\in H(a)\uparrow_ \mathfrak B H(b)$, we have
\begin{align*} 
	H(a)&=t^ \mathfrak B(\textbf{o}),\\
	H(b)&=t^ \mathfrak B(\textbf{u}),
\end{align*} for some $\textbf{o}, \textbf{u}\in B^{r(t)}$. Now since $H$ is an isomorphism, its inverse $H^{-1}$ is an isomorphism as well, and we thus have
\begin{align*} 
	a&=H^{-1}(t^ \mathfrak B(\textbf{o}))=t^ \mathfrak A(H^{-1}(\textbf{o}))\\
	b&=H^{-1}(t^ \mathfrak B(\textbf{u}))=t^ \mathfrak A(H^{-1}(\textbf{u})),
\end{align*} which shows
\begin{align*} 
	t\in a\uparrow_ \mathfrak A b.
\end{align*} Now we have by \prettyref{eq:H_downarrow_A_s=downarrow_B_s} and \prettyref{eq:t_subsetneq_s},
\begin{align*} 
	\downarrow_ \mathfrak A t=H^{-1}(\downarrow_ \mathfrak B t)\;\subsetneq\; H^{-1}(\downarrow_ \mathfrak A s)=\ \downarrow_ \mathfrak A s,
\end{align*} which is equivalent to
\begin{align*} 
	t\sqsubset_ \mathfrak A s,
\end{align*} a contradiction to \prettyref{eq:s_in_a_Uparrow_A_b} and thus to the $\sqsubseteq_ \mathfrak A$-minimality of $s$.

$(\supseteq)$ Analogous.
\end{proof}

\section{Monounary algebras}\label{§:Monounary}

In the rest of this section, let $\mathfrak A=(A,S)$ be a monounary algebra with $S:A\to A$ the only unary function (we can imagine $S$ to be a generalized ``successor'' function). 

\begin{example} We show that the generalization type of the monounary algebra
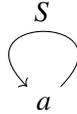
\begin{center}
\begin{tikzpicture} 
	\node (a) {$a$};
	\draw[->] (a) to [edge label'={$S$}][loop] (a);
\end{tikzpicture}
\end{center} is infinitary and trivial. We have
\begin{align*} 
	\uparrow_\mathfrak A a= \left\{S^m(x) \;\middle|\; m\geq 0\right\}.
\end{align*} Since
\begin{align*} 
	\downarrow_\mathfrak A S^m(x)=\ \downarrow_\mathfrak A S^n(x)=\{a\},\quad\text{for all $m,n\geq 0$},
\end{align*} we have
\begin{align*} 
	S^m(x)\equiv_\mathfrak A S^n(x),\quad\text{for all $m,n\geq 0$},
\end{align*} and thus
\begin{align*} 
	a\Uparrow_\mathfrak A a= \left\{S^m(x) \;\middle|\; m\geq 0\right\}=T_{\{S\}\{x\}}.
\end{align*} This shows that the generalization type of $\mathfrak A$ is infinitary and trivial.
\end{example}

\begin{example} We show that the generalization type of the monounary algebra
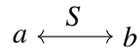
\begin{center}
\begin{tikzpicture} 
	\node (a) {$a$};
	\node (b) [right=of a] {$b$};
	\draw[<->] (a) to [edge label={$S$}] (b);
\end{tikzpicture}
\end{center} is infinitary and trivial. Since
\begin{align*} 
	\uparrow_\mathfrak B a= \left\{S^m(x) \;\middle|\; m\geq 0\right\}=\ \uparrow_\mathfrak B b
\end{align*} and
\begin{align*} 
	\downarrow_\mathfrak B S^m(x)=\ \downarrow_\mathfrak B S^n(x)=\{a,b\},\quad\text{for all $m,n\geq 0$},
\end{align*} implies
\begin{align*} 
	S^m(x)\equiv_\mathfrak B S^n(x),\quad\text{for all $m,n\geq 0$},
\end{align*} and thus
\begin{align*} 
	a\Uparrow_\mathfrak B a=b\Uparrow_\mathfrak B b=a\Uparrow_\mathfrak B b= \left\{S^m(x) \;\middle|\; m\geq 0\right\}=T_{\{S\}\{x\}}.
\end{align*}
\end{example}

We define
\begin{align*} 
	m(a)&:=
		\begin{cases}
			\max\left\{m\geq 0 \;\middle|\; S^m(x)\in\ \uparrow_\mathfrak A a\right\} & \text{if the maximum exists,}\\
			\infty & \text{otherwise,}
		\end{cases}\\
\end{align*} extended to pairs by
\begin{align*} 
	m(a,b)&:=\min(m(a),m(b))\in\mathbb N\cup\{\infty\}.
\end{align*}

% Moreover, for any $a,b\in A$ such that $m(a,b)$ exists, define
% \begin{align*} 
% 	o(a,b):=\text{select one element $o\in A$ such that $a=f^{m(a,b)}(o(a,b))$}.
% \end{align*}

\begin{theorem} Let $\mathfrak A=(A,S)$ be a monounary algebra. For any $a,b\in A$, we have
\begin{align*} 
	a\Uparrow_ \mathfrak A b=
		\begin{cases}
			\left\{S^{m(a,b)}(x)\right\} & \text{if $m(a,b)<\infty$}\\
			\emptyset & \text{otherwise}.
		\end{cases} 
\end{align*}
\end{theorem}
\begin{proof} Suppose $m(a,b)<\infty$, which means that either $m(a)<\infty$ or $m(b)<\infty$, that is, there is some maximal $m$ such that $S^m(x)\in\ \uparrow_ \mathfrak A a$ or $S^m(x)\in\ \uparrow_ \mathfrak A b$. We then have
\begin{align*} 
	S^{m+\ell}(x)\not\in (\uparrow_ \mathfrak A a)\cap (\uparrow_ \mathfrak A b)=\ a\uparrow b,\quad\text{for all $\ell\geq 1$,}
\end{align*} which implies $a\Uparrow_ \mathfrak A b=\{S^m(x)\}$ --- notice that $m=m(a,b)$ by definition. 

Now suppose $m(a,b)=\infty$, which means that $m(a)=m(b)=\infty$. In that case, we clearly have $\max_\leq(a\uparrow_ \mathfrak A b)=\emptyset$.
\end{proof}

\begin{corollary}\label{c:(N,S)_unitary} Let $(\mathbb N,S)$ be the infinite monounary algebra where $S$ denotes the successor function. For any $a,b\in\mathbb N$, we have $a\Uparrow_{(\mathbb N,S)} b=\left\{S^{\min(a,b)}(x)\right\}$. This means that the generalization type of $(\mathbb N,S)$ is unitary.
\end{corollary}

\section{Finite unary algebras aka semiautomata}\label{§:FUA}

In this section, we study algebraic anti-unification in finite unary algebras, which can be seen as semiautomata, and show that we can use well-known methods from the theory of finite automata to compute sets of (minimally general) generalizations.

In the rest of this section, let
\begin{align*} 
	\mathfrak A=(A,\Sigma:=\{f_1,\ldots,f_n\}),
\end{align*} for some $n\geq 1$, be a finite unary algebra with finite universe $A$. We shall now recall that every such algebra is essentially a semiautomaton.

Recall that a (\textit{\textbf{finite deterministic}}) \textit{\textbf{semiautomaton}} (see e.g. \cite[§2.1]{Holcombe82}) is a construct $$\mathfrak S=(S,\Sigma,\delta),$$ where $S$ is a finite set of \textit{\textbf{states}}, $\Sigma$ is a finite \textit{\textbf{input alphabet}}, and $\delta:S\times \Sigma\to S$ is a \textit{\textbf{transition function}}. Every semiautomaton can be seen as a finite unary algebra in the following well-known way: every symbol $\sigma\in\Sigma$ induces a unary function $\sigma^ \mathfrak S:S\to S$ via $\sigma^ \mathfrak S:=\delta(x,\sigma)$. We can now omit $\delta$ and define $$\mathfrak S':=(S,\Sigma^ \mathfrak S:=\{\sigma^ \mathfrak S\mid \sigma\in\Sigma\}).$$ It is immediate from the construction that $\mathfrak S$ and $\mathfrak S'$ represent essentially the same semiautomaton and that every semiautomaton can be represented in that way --- the difference is that $\mathfrak S'$ is a finite unary algebra!

Recall that a (\textit{\textbf{finite deterministic}}) \textit{\textbf{automaton}} (see e.g. \cite[§1.1]{Sipser13}) is a construct $$\mathscr A:=(Q,\Sigma,\delta,q_0,F),$$ where $(Q,\Sigma,\delta)$ is a semiautomaton, $q_0\in Q$ is the \textit{\textbf{initial state}}, and $F\subseteq Q$ is a set of \textit{\textbf{final states}}. The \textit{\textbf{behavior}} of $\mathscr A$ is given by
\begin{align*} 
	|| \mathscr A||:=\{w\in \Sigma^\ast \mid \delta^\ast(q_0,w)\in F\},
\end{align*} where $\delta^\ast:Q\times\Sigma^\ast\to Q$ is defined recursively as follows, for $q\in Q,a\in\Sigma,w\in\Sigma^\ast$:\todo{$\delta^\ast(q,\varepsilon) := q$?}
\begin{align*} 
	\delta^\ast(q,\varepsilon)&:=q,\\
	\delta^\ast(q,aw)&:=\delta^\ast(\delta(q,a),w).
\end{align*} Notice that since automata are built from semiautomata by adding an initial state and a set of final states, and since every semiautomaton $\mathfrak S=(S,\Sigma,\delta)$ can be represented in the form of a finite unary algebra $\mathfrak S'=(S,\Sigma^ \mathfrak S)$ as above, we can reformulate every automaton $\mathscr A=(Q,\Sigma,\delta,q_0,F)$ as $\mathscr A'=(Q,\Sigma^ \mathscr A,q_0,F)$, where $\Sigma^ \mathscr A:=\{\sigma^ \mathscr A\mid \sigma\in\Sigma\}$ and $\sigma^ \mathscr A:=\delta(\,.\,,\sigma):Q\to Q$. In other words, given a finite unary algebra (semiautomaton) $$\mathfrak A=(A,\Sigma),$$ we can construct a finite automaton
\begin{align*} 
	\mathfrak A_{a\to F}=(A,\Sigma,a,F)
\end{align*} by designating a state $a\in A$ as the initial state, and by designating a set of states $F\subseteq A$ as final states.

We want to compute the set of generalizations $\uparrow_\mathfrak A a$. Notice that we can identify, for example, each term in $T_{\{f,g\}}(\{x\})$ with a word over the alphabet $\Sigma=\{f,g\}$: for instance, the term $fgfx$ can be identified with the word $fgf\in\Sigma^\ast$ since the variable $x$ contains no information. We denote the function induced by a word $w\in\Sigma^\ast$ in $\mathfrak A$ by $w^ \mathfrak A$ --- for example, $(fg)^ \mathfrak A$ is the function on $A$ which first applies $g$ and then $f$. We shall now show that in any finite unary algebra (semiautomaton) $\mathfrak A=(A,\Sigma)$, the set of generalizations $\uparrow_ \mathfrak A a$ can be computed by some finite automaton as illustrated by the following example: 

\begin{example}\label{e:A} Consider the finite unary algebra (semiautomaton)
\begin{align*} 
	\mathfrak A=(\{a,b\},\Sigma:=\{f,g\})
\end{align*} given by
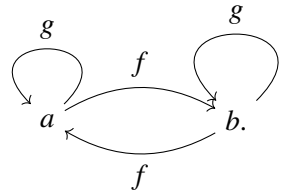
\begin{center}
\begin{tikzpicture}[node distance=2cm and 2cm] 
	\node (a) {$a$};
	\node (b) [right=of a] {$b$.};
	\draw[->] (a) to [edge label={$f$}] [bend left] (b);
	\draw[->] (b) to [edge label={$f$}] [bend left] (a);
	\draw[->] (a) to [edge label'={$g$}] [loop] (a);
	\draw[->] (b) to [edge label'={$g$}] [loop] (b);
\end{tikzpicture}
\end{center} We can identify the set of all generalizations of $a$ in $\mathfrak A$ with
\begin{align*} 
	\uparrow_\mathfrak A a=&\left\{w\in\Sigma^\ast \;\middle|\; \delta^\ast(a,w)=a\right\}\cup \left\{u\in\Sigma^\ast \;\middle|\; \delta^\ast(b,u)=a\right\}.
\end{align*} Now define the automaton $\mathfrak A_{a\to \{a\}}$ by adding to the semiautomaton $\mathfrak A$ the initial state $a$ and the set of final states $\{a\}$ (we use here the standard pictorial notation for automata)
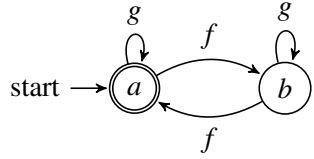
\begin{center}
\begin{tikzpicture}[->,>=stealth',semithick,shorten >= 1pt,node distance=2cm,auto]
	\node[state,initial,accepting] (a) {$a$};
	\node[state] (b) [right of=a] {$b$};
	\path[->]   
	    (a) edge [loop above] node {$g$} (a)
	    (a) edge [bend left] node {$f$} (b)
	    (b) edge [bend left] node {$f$} (a)
	    (b) edge [loop above] node {$g$} (b)
	;
\end{tikzpicture}
\end{center} and the automaton $\mathfrak A_{b\to \{a\}}$ by
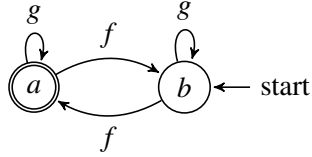
\begin{center}
\begin{tikzpicture}[->,>=stealth',semithick,shorten >= 1pt,node distance=2cm,auto,initial where=right]
	\node[state,accepting] (a) {$a$};
	\node[state,initial] (b) [right of=a] {$b$};
	\path[->]   
	    (a) edge [loop above] node {$g$} (a)
	    (a) edge [bend left] node {$f$} (b)
	    (b) edge [bend left] node {$f$} (a)
	    (b) edge [loop above] node {$g$} (b)
	;
\end{tikzpicture}
\end{center} We then clearly have
\begin{align*} 
	\uparrow_\mathfrak A a=||\mathfrak A_{a\to \{a\}}||\cup ||\mathfrak A_{b\to \{a\}}||.
\end{align*}
\end{example}

It is straightforward to generalize the construction in \prettyref{e:A}: 

\begin{definition} Given a finite unary algebra (semiautomaton) $\mathfrak A=(A,\Sigma)$, the automaton $\mathfrak A_{b\to \{a\}}$ is the automaton induced by the functions in $\Sigma$ with start state $b$ and single final state $a$ given by
\begin{align*} 
	\mathfrak A_{b\to \{a\}}:=(A,\Sigma,b,\{a\}).
\end{align*} 
\end{definition}

\begin{proposition} Given any finite\footnote{Finiteness is required since regular languages are not closed under \textit{infinite} union.} unary algebra (semiautomaton) $\mathfrak A=(A,\Sigma)$ and $a\in A$, we have
\begin{align}\label{eq:uparrow_a}
	\uparrow_\mathfrak A a=\bigcup_{b\in A} ||\mathfrak A_{b\to \{a\}}||.
\end{align}
\end{proposition}

We are now ready to prove the main result of this section:

\begin{theorem} Let $\mathfrak A=(A,\Sigma^ \mathfrak A)$ and $\mathfrak B=(B,\Sigma^ \mathfrak B)$ be finite unary algebras (semiautomata) over the same set of function symbols (input alphabet) $\Sigma$. We have the following:
\begin{enumerate}
	\item For any element (state) $a\in A$, $\uparrow_ \mathfrak A a$ is a regular language.
	\item For any elements (states) $a\in A$ and $b\in B$, $a\uparrow_{ \mathfrak{(A,B)}} b$ is a regular language.
	\item For any elements (states) $a\in A$ and $b\in B$, $a\Uparrow_{ \mathfrak{(A,B)}} b$ is computable.
\end{enumerate}
\end{theorem}
\begin{proof} Since $A$ is finite and since regular languages are known to be closed under finitely many unions, we conclude by \prettyref{eq:uparrow_a} that $\uparrow_ \mathfrak A a$ is a regular language.

Since regular languages are known to be closed under finitely many intersections and since we already know that $\uparrow_ \mathfrak A a$ and $\uparrow_ \mathfrak B b$ are regular languages, we conclude that
\begin{align*} 
	a\uparrow_{ \mathfrak{(A,B)}} b=(\uparrow_\mathfrak A a)\cap (\uparrow_\mathfrak B b)
\end{align*} is a regular language.

Given some word $w\in a\uparrow_{ \mathfrak{(A,B)}} b$ (recall that we identify the term $wx\in T_{\Sigma,\{x\}}$ with the word $w\in\Sigma^\ast$ since $x$ bears no information), we have $w\in a\Uparrow_{ \mathfrak{(A,B)}} b$ iff the term $wx\in T_{\Sigma,\{x\}}$ is $\sqsubseteq_\mathfrak A$-minimal and $\sqsubseteq_\mathfrak B$-minimal, which amounts to deciding whether there is some term $ux\in T_{\Sigma,\{x\}}$ such that $u\in a\uparrow_{ \mathfrak{(A,B)}} b$ and
\begin{align*} 
	\left\{u^ \mathfrak A(a)\in A \;\middle|\; a\in A\right\}\subsetneq \left\{w^ \mathfrak A(a)\in A \;\middle|\; a\in A\right\} \quad\text{or}\quad \left\{u^ \mathfrak B(b)\in B \;\middle|\; b\in B\right\}\subsetneq \left\{w^ \mathfrak B(b)\in B \;\middle|\; b\in B\right\}.
\end{align*} In a finite algebra, this is clearly a computable relation. Hence, $a\Uparrow_ \mathfrak A b$ is computable in finite unary algebras (semiautomata).
\end{proof}

\section{\texorpdfstring{The $(k,\ell)$-fragments}{Fragments}}\label{§:Fragments}

% Since computing the set of \textit{all} generalizations is rather difficult in general, it is reasonable to study fragments of algebraic anti-unification. For this, we introduce in this section the $(k,\ell)$-fragments:

The automata-theoretic treatment of \prettyref{§:FA} requires restricting attention to terms over a finite set of variables. This motivates the following notion of $(k,\ell)$-fragments:

\begin{definition} Let $X_k:=\{x_1,\ldots,x_k\}$, for some $k,\ell\in\mathbb N\cup\{\infty\}$ so that $X_\infty=X$. Define
\begin{align*} 
    \uparrow^{(k,\ell)}_\mathfrak A a:=(\uparrow_\mathfrak A a)\cap\{s(x_1,\ldots,x_k)\in T_{L,X_k}\mid\text{each of the $k$ variables in $X_k$ occurs at most $\ell$ times in $s$}\}.
\end{align*} We write $k$ instead of $(k,\infty)$ so that
\begin{align}\label{eq:uparrow^k_A_a} 
    \uparrow^k_ \mathfrak A a=(\uparrow_ \mathfrak A a)\cap T_{L,X_k}.
\end{align} The simplest fragment --- namely, the $(1,1)$-fragment --- contains only \textit{\textbf{monolinear generalizations}} containing exactly one occurrence of a single variable $x$. We denote the \textit{\textbf{monolinear generalization}} and \textit{\textbf{instantiation operations}} in $\mathfrak A$ by $\uparrow^m_\mathfrak A$ and $\downarrow_{ \mathfrak A,m}$, respectively, and the so-obtained anti-unification relation by $\Uparrow_{ \mathfrak A,m}$.
\end{definition}

As a simple demonstration of fragments, we compute monolinear algebraic anti-unification in the set domain:

\begin{proposition} For any universe $U$ and $A,B\subseteq U$, we have
\begin{align*} 
    A\Uparrow^m_{(2^U,\cup,2^U)} B&=
            \begin{cases}
                \{X\cup (A\cap B)\} & A\neq B,\\
                \{A\} & A=B,
            \end{cases}\\
    A\Uparrow^m_{(2^U,\cap,2^U)} B&=
            \begin{cases}
                \{X\cap (A\cup B)\} & A\neq B,\\
                \{A\} & A=B,
            \end{cases}\\
    A\Uparrow^m_{(2^U,.^c,2^U)} B&=
            \begin{cases}
                \{X,X^c\} & A\neq B,\\
                \{A\} & A=B.
            \end{cases}
\end{align*}
\end{proposition}
\begin{proof} Given two sets $C,D\subseteq U$, we define
\begin{align*} 
    [C,D]:=\{E\subseteq U\mid C\subseteq E\subseteq D\}.
\end{align*}
\begin{enumerate}
    \item First, we work in $(2^U,\cup,2^U)$ and omit the explicit reference to the algebra. We have
    \begin{align*} 
        A\uparrow^m B=\{X\cup C\mid\emptyset\subseteq C\subseteq A\cap B\}\cup\{A\mid \text{if $A=B$}\}
    \end{align*} and
    \begin{align*} 
        \downarrow_m (X\cup C)=[C,U].
    \end{align*} This implies
    \begin{align*} 
        X\cup C\sqsubseteq X\cup D \quad\Leftrightarrow\quad [C,U]\subseteq [D,U] \quad\Leftrightarrow\quad D\subseteq C.
    \end{align*} Hence
    \begin{align*} 
        A\Uparrow^m B=
            \begin{cases}
                \{X\cup (A\cap B)\} & A\neq B\\
                \{A\} & A=B.
            \end{cases}
    \end{align*}

    \item Second, we work in $(2^U,\cap,2^U)$ and omit the explicit reference to the algebra. We have
    \begin{align*} 
        A\uparrow^m B=\{X\cap C\mid A,B\subseteq C\}\cup\{A\mid \text{if $A=B$}\}
    \end{align*} and
    \begin{align*} 
        \downarrow_m(X\cap C)=[\emptyset,C].
    \end{align*} This implies
    \begin{align*} 
        X\cap C\sqsubseteq X\cap D \quad\Leftrightarrow\quad [\emptyset,C]\subseteq [\emptyset,D] \quad\Leftrightarrow\quad C\subseteq D.
    \end{align*} Hence
    \begin{align*} 
        A\Uparrow^m B=
            \begin{cases}
                \{X\cap (A\cup B)\} & A\neq B\\
                \{A\} & A=B.
            \end{cases}
    \end{align*}

    \item Third, we work in $(2^U,.^c,2^U)$ and omit the explicit reference to the algebra. We have
    \begin{align*} 
        A\uparrow^m B=\{X,X^c\}\cup\{A\mid \text{if $A=B$}\}
    \end{align*} and
    \begin{align*} 
        \downarrow_m X^c=2^U=\ \downarrow X.
    \end{align*} This implies
    \begin{align*} 
        X^c\equiv X.
    \end{align*} Hence
    \begin{align*} 
        A\Uparrow^m B=
            \begin{cases}
                \{X,X^c\} & A\neq B\\
                \{A\} & A=B.
            \end{cases}
    \end{align*}
\end{enumerate}
\end{proof}

% \begin{remark} In the full set algebra $(2^U,\cup,\cap,c,2^U)$, it already appears challenging to compute $A\Uparrow^m B$, for given sets $A,B\in 2^U$ (see \prettyref{probl:sets}).
% \end{remark}

\section{Finite algebras}\label{§:FA}

In this section, let $\mathfrak A$ be a \textit{finite} algebra which means that its underlying universe $A$ is a finite set and it contains finitely many functions. For $k\geq 1$, let $X_k:=\{x_1,\ldots,x_k\}$. 

Recall that a (\textit{\textbf{frontier-to-root}}) \textit{\textbf{tree automaton}} (see e.g. \cite{Gecseg15})
\begin{align*} 
    \mathscr T_{k,\alpha,F}(\mathfrak A):=(\mathfrak A,L,X_k,\alpha,F)
\end{align*} consists of a finite $L$-algebra $\mathfrak A$, an \textit{\textbf{initial assignment}} $\alpha:X_k\to A$, and a set $F\subseteq A$ of \textit{\textbf{final states}}. 

The \textit{\textbf{regular tree language}} recognized by $\mathscr T_{k,\alpha,F}(\mathfrak A)$ is given by
\begin{align*} 
	||\mathscr T_{k,\alpha,F}(\mathfrak A)||:=\left\{s\in T_{L,X_k} \;\middle|\; s^\mathfrak A\alpha\in F\right\}.
\end{align*}

We have (recall the definition of $\uparrow_\mathfrak A^k a$ from \prettyref{eq:uparrow^k_A_a})
\begin{align*} 
	\uparrow_\mathfrak A^k a=\bigcup_{\alpha\in A^{X_k}}||\mathscr T_{k,\alpha,\{a\}}(\mathfrak A)||.
\end{align*} Since $A^{X_k}$ is a finite set and tree automata are closed under finite union, the set $\uparrow_\mathfrak A^k a$ is a regular tree language. Moreover, since tree automata are closed under finite intersection, there is some tree automaton $\mathfrak T_{k,a,b}(\mathfrak A)$ such that
\begin{align*} 
	a\uparrow^k_\mathfrak A b=(\uparrow_\mathfrak A^k a)\cap (\uparrow_\mathfrak A^k b)=||\mathfrak T_{k,a,b}(\mathfrak A)||.
\end{align*} 

To compute the set of \textit{\textbf{minimally general $k$-generalizations}} $a\Uparrow^k_ \mathfrak A b$ it therefore only remains to check for each $s\in a\uparrow^k_\mathfrak A b=||\mathfrak T_{k,a,b}(\mathfrak A)||$ whether $s$ is $\sqsubseteq_ \mathfrak A$-minimal among the $k$-generalizations of $a$ and $b$ in $\mathfrak A$.\todo{Ist das berechenbar?}

\section{Conclusion}

This paper introduced algebraic anti-unification in the general setting of universal algebra and thereby extends classical anti-unification beyond term algebras and equational theories. We introduced the notions of algebraic generalization ordering and minimally general generalization, established basic structural properties, proved compatibility with homomorphisms and isomorphisms, and obtained computability results for finite unary algebras via finite automata. More generally, we argued that tree automata provide a natural framework for studying algebraic anti-unification in finite algebras. In this way, the present work complements the purely syntactic theory of anti-unification initiated in the seminal works of \cite{Plotkin70} and \cite{Reynolds70}.

A major line of future research is to make progress towards the following fundamental problem:

\begin{problem}
Given two $L$-algebras $\mathfrak A$ and $\mathfrak B$, what is the generalization type of $\mathfrak{(A,B)}$? For what kind of algebras is the problem decidable? In those cases where it is decidable, provide an (efficient) algorithm.
\end{problem}

The framework of this paper is unilingual in the sense that the underlying languages of the involved algebras are the same. This is common practice in universal algebra. However, many applications in artificial intelligence involve the comparison, integration, or transfer of knowledge between systems employing different representational vocabularies. A major line of future theoretical research therefore is to generalize the notions and results of this paper from a unilingual to a bilingual setting where the underlying languages $L_{\mathfrak A}$ and $L_{\mathfrak B}$ of $\mathfrak A$ and $\mathfrak B$ may differ. One promising possibility is to use the well-known notion of interpretability of one theory into another (see e.g.\ \cite[\S2.6]{Hinman05}), which applies naturally here since an algebra can be viewed as a logical structure with equality as its sole relation.

From a practical point of view, the main line of future research is to study computability and complexity issues. Recall that in \prettyref{§:FUA}, we have shown that in finite unary algebras (aka semiautomata), minimally general generalizations can be computed using standard techniques from the theory of finite automata. Moreover, in \prettyref{§:FA}, we have shown that more generally, in any finite algebra, tree automata provide a natural framework for studying minimally general $k$-generalizations, for any fixed $k$. This is closely related to finite model theory (see e.g.\ \cite{Ebbinghaus99,Libkin12}). However, since in practice we often encounter finitely representable \textit{infinite} structures \cite{Blumensath00,Blumensath04}, obtaining analogous computability results is an important challenge.

Another important line of applied future research is to develop \textit{efficient} algorithms for the computation of minimally general generalizations in finite and infinite structures and to provide implementations which can be used in practice.

Finally, the original motivation for this work comes from recent applications of generalization to similarity \cite{Antic23-2} and analogical proportions \cite{Antic22,Antic23-22}. Understanding the precise role of minimally general generalizations in these settings remains an interesting open problem.

\if\isdraft1\newpage\section*{*}\fi
\bibliographystyle{acm}
\bibliography{/Users/christianantic/Bibdesk/Bibliography,/Users/christianantic/Bibdesk/Publications_J,/Users/christianantic/Bibdesk/Publications_C,/Users/christianantic/Bibdesk/Preprints,/Users/christianantic/Bibdesk/Submitted,/Users/christianantic/Bibdesk/Drafts,/Users/christianantic/Bibdesk/Bin}
\if\isdraft1\newpage

\section{}

Now we define the operation $\uparrow$ dual to the instantiation operation $\downarrow$ of \prettyref{d:downarrow}:

\begin{definition} For any term $t\in T_{L,X}$, define
\begin{align*} 
	\uparrow_{ \mathfrak{(A,B)}} t:=\{s\in T_{L,X} \mid t\sqsubseteq_{ \mathfrak{(A,B)}} s\}.
\end{align*} In case $s\in\ \uparrow_{ \mathfrak{(A,B)}} t$, we say that $s$ is a \textit{\textbf{generalization}} of $t$ in $\mathfrak{(A,B)}$.
\end{definition}

Recall that a \textit{\textbf{pre-filter}} $F$ on a pre-ordered set $(P,\leq)$ is a subset of $P$ satisfying:
\begin{enumerate}
	\item $F$ is non-empty.
	\item $F$ is downward directed, that is, for every $a,b\in F$, there is some $c\in F$ such that $c\leq a,b$.
	\item $F$ is an upper set or upward closed, that is, for every $a\in F$ and $b\in P$, if $a\leq b$ then $b\in F$.
\end{enumerate} The smallest pre-filter containing an element $a$ is a \textit{\textbf{principal pre-filter}} and $a$ is a \textit{\textbf{principal element}} --- it is given by
\begin{align*} 
	\uparrow_{(P,\leq)} a:=\{b\in P\mid a\leq b\}.
\end{align*}

The following observation motivates our notation using the symbol $\uparrow$:

\begin{fact} The set $\uparrow_{ \mathfrak{(A,B)}} t$ is the principal pre-filter with respect to the generalization ordering $\sqsubseteq_{ \mathfrak{(A,B)}}$ generated by $t$.
\end{fact}

\section{}

In \prettyref{§:GAAU} haben wir $A\Uparrow B$, fuer Mengen $A,B$, betrachtet. Da jeder Term $t$ mit der \textit{Menge} $\downarrow t$ identifiziert werden kann, haben wir damit automatisch auch $r\Uparrow t$ definiert, das der klassischen AU von Termen aehnlich ist.

\section{}

\begin{definition} For any terms $s,t\in T_{L,X}$, define
\begin{align*} 
	s\subseteq_{ \mathfrak{(A,B)}} t \quad:\Leftrightarrow\quad \uparrow_\mathfrak A s\subseteq\ \uparrow_\mathfrak A t.
\end{align*}
\end{definition}

\begin{fact} $\uparrow a\subseteq\ \uparrow b$ implies $a\Uparrow b=a\Uparrow a$.
\end{fact}
\begin{proof} $\uparrow a\subseteq\ \uparrow b$ implies $a\uparrow b=\ \uparrow a$ implies $\min_\sqsubseteq(a\uparrow b)=\min_\sqsubseteq\uparrow a$ implies $a\Uparrow b=a\Uparrow a$.
\end{proof}

\section{}

Define % wir muessen hier $\bigcup$ verwenden, da bei $\bigcap$ sobald es $t,t'\in a\Uparrow b$ gibt, $t\neq t'$, kann schon $a\Uparrow b\sqsubseteq t$ nicht gelten!
\begin{align*} 
	\uparrow_{ \mathfrak{(A,B)}}(a\Uparrow_{ \mathfrak{(A,B)}} b):=\bigcup_{t\in a\Uparrow_{ \mathfrak{(A,B)}} b}\uparrow_{ \mathfrak{(A,B)}} t=\bigcup_{t\in a\Uparrow_{ \mathfrak{(A,B)}} b}\{s\in T_{L,X}\mid t\sqsubseteq_{ \mathfrak{(A,B)}} s\}.
\end{align*} Notice that we have
\begin{align}\label{eq:uparrow_emptyset} 
	\uparrow_{ \mathfrak{(A,B)}}\emptyset=\emptyset.
\end{align}

\begin{fact} $a\Uparrow_{ \mathfrak{(A,B)}} b\neq\emptyset$ implies $a\uparrow_{ \mathfrak{(A,B)}} b=\ \uparrow_{ \mathfrak{(A,B)}}(a\Uparrow_{ \mathfrak{(A,B)}} b)$.
\end{fact}
\begin{proof} ($\subseteq$)

$(\supseteq)$ Trivial.
\todo[inline]{}
\end{proof}
\begin{proof}
$(\subseteq)$ We distinguish two cases:
\begin{description}
	\item[$a\Uparrow_{ \mathfrak{(A,B)}} b=\emptyset$] By \prettyref{eq:uparrow_emptyset}, we need to show $a\uparrow_{ \mathfrak{(A,B)}} b=T_{L,X}$. Suppose there is some $s\in T_{L,X}$ such that $s\not\in a\uparrow_{ \mathfrak{(A,B)}} b$. This is equivalent to $s\not\in\ \uparrow_ \mathfrak A a$ or $s\not\in\ \uparrow_ \mathfrak B b$. Without loss of generality, we assume $s\not\in\ \uparrow_ \mathfrak A a$.

	\item[$a\Uparrow_{ \mathfrak{(A,B)}} b=\emptyset$] Let $s\in a\uparrow_{ \mathfrak{(A,B)}} b$...
\end{description} 

$(\supseteq)$
\todo[inline]{}
\end{proof}

\section{}

\begin{fact}\label{f: 230522-uparrow_a=z} $\uparrow a=\{x\}$ implies $a\sqcap b=\{x\}$, for all $b\in A$.
\end{fact}

\begin{example} In the monounary algebra $\mathfrak A=(\{a,b\},f)$ given by
\begin{center}
\begin{tikzpicture} 
	\node (a) {$a$};
	\node (b) [right=of a] {$b$};
	\draw[->] (a) to [edge label={$f$}] (b);
	\draw[->] (b) to [edge label'={$f$}][loop] (b);
\end{tikzpicture}
\end{center} we have
\begin{align*} 
	\uparrow_\mathfrak A a = \{x\}
\end{align*} and thus by \prettyref{f: 230522-uparrow_a=z}
\begin{align*} 
	a\sqcap b = \{x\}.
\end{align*}
\end{example}

\begin{fact} $\uparrow a\subseteq\ \uparrow b$ implies $a\sqcap b=a\sqcap a$.
\end{fact}
\begin{proof} $\uparrow a\subseteq\ \uparrow b$ implies $\uparrow(a,b)=\ \uparrow a$ implies $\min_\sqsubseteq\uparrow(a,b)=\min_\sqsubseteq\uparrow a$ implies $a\sqcap b=a\sqcap a$.
\end{proof}

\section{Minors}

Recall that $f$ is a {\em minor} of $g$ --- in symbols, $f\leq_m g$ --- iff $f$ can be obtained from $g$ by identification of arguments, permutation of arguments, or introduction or deletion of inessential arguments. Formally, let $[n]:=\{1,\ldots,n\}$ and let $\alpha:[r(g)]\to [r(f)]$ be a map, for some functions $f,g$ with $r(f)\leq r(g)$. Then $f$ is a minor of $g$ iff
\begin{align}\label{eq:f_g_alpha} 
	f(a_1,\ldots,a_{r(f)}) = g(a_{\alpha(1)},\ldots,a_{\alpha(r(g))}).
\end{align}

\begin{fact} If $f\leq_m g$ then $\downarrow f\subseteq\ \downarrow g$.
\end{fact}

\section{Injective term functions}

\begin{fact}\label{f: 230902-injective} Term functions consisting only of injective functions are injective.
\end{fact}
\begin{proof} By structural induction on the shape of a term $s$. Atomic terms are not injective. It remains to show that whenever $s_1,\ldots,s_{rf}$ are injective terms and $f$ is an injective function, then $s:=f(s_1,\ldots,s_{rf})$ is an injective term function as well --- this holds trivially.
\end{proof}

\section{An algorithm for anti-unification in general algebras (not working)}

Define the {\em head} of a term $s$ inductively as follows:
\begin{itemize}
	\item If $s$ is a variable then $h(s):=s$.
	\item If $s=f(s_1,\ldots,s_{r(f)})$ then $h(s):=f$.
\end{itemize} Notice that we have
\begin{align*} 
	h(\uparrow_\mathfrak A a) = \{h(s)\mid s\in\ \uparrow_\mathfrak A a\}=\left\{f\in L \;\middle|\; a=f^\mathfrak A(\mathbf o),\mathbf o\in A^{r(f)}\right\}
\end{align*} and
\begin{align*} 
	a\in h(\uparrow_\mathfrak A a) \quad\text{iff}\quad a\in L_0.
\end{align*}

For sets of terms $S_1,\ldots,S_{r(f)}$, we define
\begin{align*} 
	f(S_1,\ldots,S_{r(f)}):=\{f(s_1,\ldots,s_{r(f)})\mid s_i\in S_i,i\in [r(f)]\}.
\end{align*}

\begin{pseudocode}\label{pseudocode:Plotkin} We now generalize \cite{Plotkin70} algorithm and provide an algorithm for the computation of $a\sqcap_\phi b$ with respect to an injective function $\phi:A\times A\to Z$:
\begin{enumerate}
	\item If $h(\uparrow a)\cap h(\uparrow b)=\emptyset$, then $a\sqcap_\phi b:=\{\phi(a,b)\}$.
	\item Otherwise, for every $f\in h(\uparrow a)\cap h(\uparrow b)$:
	\begin{prooftree}
		\AxiomC{$a=f^\mathfrak A(a_1,\ldots,a_{r(f)})$}
		\AxiomC{$b=f^\mathfrak A(b',\ldots,b_{r(f)})$}
		\BinaryInfC{$f(a_1\sqcap_\phi b',\ldots,a_{r(f)}\sqcap_\phi b_{r(f)})\subseteq a\sqcap_\phi b$.}
	\end{prooftree}
\end{enumerate}
\todo[inline]{funktioniert so nicht!}
\end{pseudocode} 

\prettyref{pseudocode:Plotkin} is a nondeterministic algorithm and we say that it is {\em terminating} iff it produces a (possibly empty) set of terms on each\todo{``some''?} branch. Notice that we have\todo{check}
\begin{align*} 
	a\in a\sqcap_\phi a=a(\;)\sqcap_\phi a(\;) \quad\text{iff}\quad a\in L_0.
\end{align*}

\begin{example} Let $\mathfrak A:=(\mathbb N_1,+,\mathbb N_1)$ and let $\phi:\mathbb N^2\to Z$ be an injective mapping. If $a=b$, then $a\sqcap_\phi b=a\sqcap_\phi a=\{a\}$ --- in particular, we have $1\sqcap_\phi 1=\{1\}$. Since $1=1+x$ is unsolvable in $\mathbb N_1$, we have $h(\uparrow_\mathfrak A 1)=\emptyset$ and thus
\begin{align*} 
	1\sqcap_\phi b=\{\phi(1,b)\}
\end{align*} for all $b\in\mathbb N$. We now want to compute $2\sqcap_\phi b$, $b\neq 2$:
\begin{prooftree}
	\AxiomC{$2=1+1$}
	\AxiomC{$b=b'+(b-b')$ such that $b'\neq 1$}
	\BinaryInfC{$\{\phi(1,b')+\phi(1,b-b')\}\subseteq 2\sqcap_\phi b$}
\end{prooftree} and
\begin{prooftree}
	\AxiomC{$2=1+1$}
	\AxiomC{$b=1+(b-1)$}
	\BinaryInfC{$\{1+\phi(1,b-1)\}\subseteq 2\sqcap_\phi b$.}
\end{prooftree} Since $2=1+1$ is the only possible decomposition of $2$ in $\mathbb N_1$ and $b\neq 2$ by assumption, the above two derivation schemes are the only possible ones and we thus have the output:
\begin{align*} 
	S=\{1+\phi(1,b-1),\phi(1,b')+\phi(1,b-b')\},
\end{align*} Notice that we have
\begin{align*} 
	1+\phi(1,b-1)\equiv_\mathfrak A\phi(1,b')+\phi(1,b-b')
\end{align*} since
\begin{align*} 
	\downarrow_\mathfrak A(1+\phi(1,b-1))=\ \downarrow_\mathfrak A(\phi(1,b')+\phi(1,b-b'))=\{2,3,\ldots\},
\end{align*} which shows that $S$ is indeed a set of incomparable terms with respect to the generalization ordering.
\end{example}

\begin{theorem} \prettyref{pseudocode:Plotkin} is sound and complete:
\begin{description}
	\item[Soundness] If $s$ is produced by \prettyref{pseudocode:Plotkin}, then $s\in a\sqcap_\phi b$.
	\item[Completeness] If $s\in a\sqcap_\phi b$, then $s$ is produced by \prettyref{pseudocode:Plotkin}.
\end{description}
\end{theorem}
\begin{proof} We first prove soundness and assume that $s$ is computed by \prettyref{pseudocode:Plotkin}. We show by structural induction on the shape of $s$ that $s\in a\sqcap_\phi b$. If $s$ is a variable, then it is computed by \prettyref{pseudocode:Plotkin} iff $h(\uparrow a)\cap h(\uparrow b)=\emptyset$. It is clear that in that case there can be no less general generalization of $a$ and $b$, which means $s\in a\sqcap_\phi b$. If $s$ is an element (i.e. constant symbol) $c$, then $s$ is computed by \prettyref{pseudocode:Plotkin} iff $a=b=c$, in which case $a\sqcap_\phi b=\{a\}$ and thus $s\in a\sqcap_\phi b$. Otherwise, if $s=f(s_1,\ldots,s_{r(f)})$, then $s$ is computed by \prettyref{pseudocode:Plotkin} iff $a=f^\mathfrak A(a_1,\ldots,a_{r(f)})$ and $b=f^\mathfrak A(b_1,\ldots,b_{r(f)})$, for some $a_i,b_i\in A$, $i\in [r(f)]$, such that $s_i\in a_i\sqcap_\phi b_i$ for all $i\in [r(f)]$. Then $s$ is clearly a generalization of $a$ and $b$ and thus en element of $\uparrow(a,b)$. It remains to show that $s$ is minimal with respect to the generalization ordering, that is, there is no $t\in\ \uparrow(a,b)$ such that $\downarrow t\subsetneq\ \downarrow s$.
\todo[inline]{}
\end{proof}

\todo[inline]{When do we have $|a\sqcap b|=0,1,\infty$?}

\begin{notation}\label{not: 230518-s} In those cases where $a\sqcap b=\{s\}$ is a singleton, we identify $a\sqcap b$ with its single element $s$.
\end{notation}

\begin{proposition} $a\sqcap b=\emptyset$ iff \prettyref{pseudocode:Plotkin} does not terminate.
\end{proposition}
\begin{proof} 
\todo[inline]{}
\end{proof}

\section{A naive adaptation of Plotkin's anti-unification algorithm for words does not work}

The following naive adaptation of Plotkin's algorithm does {\em not} work:

\begin{pseudocode}\label{pseudocode: 230523-a_sqcap_b} We work in the word algebra $(A^+,\cdot,A^+)$ for some alphabet $A$. Let $\phi:A^+\times A^+\to Z$ be an injective mapping. We define $\mathbf a\sqcap_\phi\mathbf b$ inductively as follows:
\begin{itemize}
	\item For any word $\mathbf a\in A^+$, we define $\mathbf a\sqcap\mathbf a:=\{\mathbf a\}$.
	\item For any letter $a\in A$ and word $\mathbf b\in A^+$, we define\begin{align*} 
		a\sqcap\mathbf b:=
			\begin{cases}
				\{a\} & a=\mathbf b,\\
				\{\phi(a,\mathbf b)\} & a\neq\mathbf b.
			\end{cases}
	\end{align*}
	\item For any words $\mathbf{a,b}\in A^+$, we define
	\begin{prooftree}
		\AxiomC{$\mathbf a=\mathbf a_1\mathbf a_2$}
		\AxiomC{$\mathbf b=\mathbf b_1\mathbf b_2$}
		\BinaryInfC{$(\mathbf a_1\sqcap_\phi\mathbf b_1)(\mathbf a_2\sqcap_\phi\mathbf b_2)\in\mathbf a\sqcap_\phi\mathbf b$.}
	\end{prooftree}
\end{itemize}
\end{pseudocode}

The following counterexample shows that \prettyref{pseudocode: 230523-a_sqcap_b} does {\em not} work:

\begin{example} The word $abc$ can be decomposed as $a\cdot bc$ and $ab\cdot c$, which with the above algorithm yields $a\phi(bc,b),\phi(ab,a)\phi(c,b)\in abc\sqcap_\phi ab$. This contradicts the fact that
\begin{align*} 
	a\phi(bc,b)\sqsubset\phi(ab,a)\phi(c,b).
\end{align*}
\end{example}

What about the following variant of the above algorithm:

\begin{pseudocode}\label{pseudocode: 230523-a_sqcap_b_2} We define $\mathbf a\sqcap_\phi\mathbf b$ inductively as follows:
\begin{itemize}
	\item For any word $\mathbf a\in A^+$, we define $\mathbf a\sqcap\mathbf a:=\{\mathbf a\}$.
	\item For any letter $a\in A$ and word $\mathbf b\in A^+$, we define
	\begin{align*} 
		a\sqcap\mathbf b:=
			\begin{cases}
				\{a\} & a=\mathbf b,\\
				\{\phi(a,\mathbf b)\} & a\neq\mathbf b.
			\end{cases}
	\end{align*}
	\item For any words $\mathbf{a,b}\in A^+$, we define
	\begin{prooftree}
		\AxiomC{$\mathbf a=a_1\mathbf a_2$}
		\AxiomC{$\mathbf b=b_1\mathbf b_2$}
		\BinaryInfC{$(a_1\sqcap_\phi b_1)(\mathbf a_2\sqcap_\phi\mathbf b_2)\in\mathbf a\sqcap_\phi\mathbf b$.}
	\end{prooftree}
\end{itemize}
\end{pseudocode}

Notice that \prettyref{pseudocode: 230523-a_sqcap_b_2} yields a {\em single} generalization for any pair of words.

\begin{example} We obtain 
\begin{align*} 
	abc\sqcap_\phi ab=\{a\phi(bc,b)\}
\end{align*} and
\begin{align*} 
	abcd\sqcap_\phi abdd=\{ab\phi(c,d)d\}
\end{align*} which appears to be correct answers.
\end{example}

\begin{claim} \prettyref{pseudocode: 230523-a_sqcap_b_2} is sound and complete.
\end{claim}
\begin{proof} We first prove soundness by showing that whenever \prettyref{pseudocode: 230523-a_sqcap_b_2} computes $s$ for given input words $\mathbf a$ and $\mathbf b$, we indeed have $s\in\mathbf a\sqcap_\phi\mathbf b$. That $s$ is always a generalization of $\mathbf a$ and $\mathbf b$ holds trivially. It remains to show that it is $\sqsubseteq$-minimal. We proceed by induction on the length of $\mathbf a$. The induction base where the lenght of $\mathbf a$ is one holds trivially. Now suppose $\mathbf a=a_1\ldots a_{n+1}$ has length $n+1$, for some $n\geq 1$...
\todo[inline]{}
\end{proof}

\section{Anti-unification in injective cycle-free algebras}

Wir versuchen \cite{Huet76} einfachen AU-Algorithmus aus der Termalgebra auf beliebige injektive Algebren zu verallgemeinern.

Sei $\mathfrak A=(A,f_1,\ldots,f_n)$ eine injektive Algebra, dh $f_1,\ldots,f_n$ sind injektive Funktionen.\footnote{Notice that every distinguished element $a$ is an 0-ary function $a:A^0\to A$ which maps some single dummy element in $A^0$ to $a$ and thus \textit{is} injective.} Wir wollen $a\Uparrow b$ berechnen. Beachte, dass es in injektiven Algebren nur maximal \textit{endlich} viele Repraesentationen von $a$ und $b$ der Form
\begin{align*} 
	a&=f_{i_1}(a_1^1,\ldots,a_{rf_1}^1)=\ldots=f_{i_k}(a_1^k,\ldots,a_{r(f_k)}^k),\\
	b&=f_{i_1}(b_1^1,\ldots,b_{rf_1}^1)=\ldots=f_{i_k}(b_1^k,\ldots,b_{r(f_k)}^k),
\end{align*} $1\leq i_1 \leq \ldots \leq i_k\leq n$, geben kann --- denn fuer jedes der \textit{endlich} vielen $f_j$, $j\in \{i_1,\ldots,i_k\}$, existiert jeweils maximal \textit{ein} Tupel $\textbf{a}$ sodass $a=f_j(\textbf{a})$ und maximal ein Tupel $\textbf{b}$ sodass $b=f_j(\textbf{b})$. 

Wir fuehren nun folgende naheliegende, an \cite{Huet76} angelehnte, Regel ein, wobei $\chi:A\times A\to X$ eine beliebige injektive Funktion sei und $1\leq i\leq n$:
\begin{align}\label{eq: 230823-Huet76} 
	f_i(a_1^i,\ldots,a_{r(f_i)}^i)\Uparrow_\chi f_i(b_1^i,\ldots,b_{r(f_i)}^i)& := f_i(a_1^i\Uparrow_\chi b_1^i,\ldots,a_{r(f_i)}^i\Uparrow_\chi b_{r(f_i)}^i).
	% f_i(a_1,\ldots,a_{r(f_i)})\Uparrow_\chi f_j(b_1,\ldots,b_{r(f_i)})&:=\chi(f_i(a_1,\ldots,a_{r(f_i)}),f_j(b_1,\ldots,b_{r(f_i)})).
\end{align} Es bleibt zu zeigen, dass da alle $f_i$'s injektiv sind, die rekursive Anwendung der Regel irgendwann stoppt und einen Term aus $a\Uparrow_\chi b$ zurueckgibt --- und dass man auf diese Art \textit{alle} Terme in $a\Uparrow_\chi b$ berechnen kann und keine weiteren. Was passiert, wenn es fuer ein $a^i_j$ (oder $b^i_j$) kein $\textbf{c}$ mehr gibt, sodass $a^i_j=f_k(\textbf{c})$, fuer ein $1\leq k\leq n$? Dann stoppt die Rekursion und wir wenden $\chi$ an!  Dieser Teil fehlt oben noch --- d.h. wenn es weder fuer $a$ noch fuer $b$ ein $f_i$ und $\textbf{c}$ gibt, sodass $a=f_i(\textbf{c})$ oder $b=f_i(\textbf{c})$, dann gib $\chi(a,b)$ zurueck.

Beachte, dass in \prettyref{eq: 230823-Huet76} auf der rechten Seite ein Ausdruck der Form
\begin{align*} 
	f_i(T_1,\ldots,T_{r(f_i)})
\end{align*} steht, wobei $T_1,\ldots,T_{r(f_i)}$ Mengen von Termen sind, die prinzipiell auch leer sein koennten --- wir zeigen, dass das nicht der Fall sein kann:

\begin{lemma} $a\Uparrow b\neq\emptyset$ gilt fuer alle Elemente einer injektiven Algebra.
\end{lemma}
\begin{proof} Es gilt $a\Uparrow b=\emptyset$ gdw. es zu jedem $s\in a\uparrow b$ ein $t\in a\uparrow b$ gibt mit $t\sqsubset s$, was aequivalent ist zu $\downarrow t\subset\ \downarrow s$. Aufgrund der Injektivitaet von $t$, muss jedes $t(\textbf{o})\in\ \downarrow t$ von jeder anderen Instanz $t(\textbf{o}')\in\ \downarrow t$, $\mathbf{o\neq o'}$, verschieden sein, was bedeutet dass wir
\begin{align*} 
	|\downarrow t\ |=|\{t(\textbf{o})\mid\textbf{o}\in A^{r(t)}\}|=|A^{r(t)}|
\end{align*} und, analog,
\begin{align*} 
	|\downarrow s\ |=|A^{r(s)}|
\end{align*} haben. Da $\downarrow t\subset\ \downarrow s$ gilt, muss auch
\begin{align*} 
	|\downarrow t\ |<|\downarrow s\ |
\end{align*} und somit
\begin{align*} 
	|A^{r(t)}|<|A^{r(s)}|
\end{align*} und letztenendes
\begin{align*}
	r(t) < r(s)
\end{align*} gelten. Da der Rang jeder Funktion endlich ist, gilt natuerlich $r(s)<\infty$ und es kann daher nur endlich viele $t$'s mit $r(t)<r(s)$ geben. Somit muss es ein $t'\in a\uparrow b$ geben, fuer das es kein $t''\in a\uparrow b$ mehr geben kann mit $t''\sqsubset t'$ --- somit ist $t'\in a\Uparrow b$ und es gilt $a\Uparrow b\neq\emptyset$.
\end{proof}

Fuer nichtleere Mengen $T_1,\ldots,T_{r(f_i)}$, definieren wir
\begin{align*} 
	f_i(T_1,\ldots,T_{r(f_i)}) := \{f_i(t_1,\ldots,t_{r(f_i)})\mid t_j\in T_j, 1\leq j\leq r(f_i)\}.
\end{align*} Es gilt nun zu zeigen, dass am Schluss der Berechnung anhand der oben eingefuehrten Regeln jedes $t=f_i(t_1,\ldots,t_{r(f_i)})\in f_i(T_1,\ldots,T_{r(f_i)})$ auch in $a\Uparrow b$ liegt... % Angenommen, es gibt ein $t'=f_i(t_1',\ldots,t_{r(f_i)}')\in a\Uparrow b$ sodass $t'\sqsubset t$...

\vspace*{0.5cm}

Folgendes Gegenbsp. zeigt, warum wir zusaetzlich \textbf{Zyklenfreiheit} verlangen muessen:

\begin{example} Betrachte die einfache Algebra
\begin{center}
\begin{tikzpicture} 
	\node (a) {$a$};
	\node (b) [right=of a] {$b$.};
	\draw[<->] (a) to [edge label={$f$}] (b);
\end{tikzpicture}
\end{center} Wollten wir nun $a\Uparrow b$ mit der Regel in \prettyref{eq: 230823-Huet76} berechnen, wuerden wir folgende $\infty$-Ableitung erhalten:
\begin{align*} 
	a\Uparrow b=f(b)\Uparrow f(a)=f(b\Uparrow a)=f(f(a)\Uparrow f(b))=f(f(a\Uparrow b))=\ldots.
\end{align*}
\end{example}

Ein noch einfacheres Gegenbsp. ist
\begin{center}
\begin{tikzpicture} 
	\node (a) {$a$};
	\draw[->] (a) to [edge label'={$f$}][loop] (a);
\end{tikzpicture} 
\end{center} wo das gleiche Phenomaen auftaucht.

Die beiden Gegenbsp. haben gemeinsamen, dass die AU infinitaer ist. Das ist wohl auf die Zyklen zurueckzufuehren. Damit der Algorithmus funktioniert, muessen wir Zyklen in (injektiven) Algebren erkennen koennen:

\begin{definition} Wir definieren den \textit{\textbf{dependency graph}} von $\mathfrak A$ als $dep(\mathfrak A):=(A,\to^\mathfrak A)$, wobei
\begin{align*} 
	a\to^\mathfrak A b \quad:\Leftrightarrow\quad b=f(\ldots,a,\ldots)
\end{align*} fuer eine Funktion $f\in L^ \mathfrak A$. Wir sagen, dass $\mathfrak A$ \textit{\textbf{zyklenfrei}} ist gdw. $dep(\mathfrak A)$ keine Zyklen enthaelt.
\end{definition}

\begin{remark} Beachte, dass eine monounaere Algebra nur dann zyklenfrei sein kann, wenn sie unendlich groß ist --- so zB in $(\mathbb N,S)$ und $(\mathbb N,P)$.
\end{remark}

Sollen wir tatsaechlich auch Zyklen zaehlen, an denen unterschiedliche Funktionen beteiligt sind? Betrachte die Algebra
\begin{center}
\begin{tikzpicture} 
	\node (a) {$a$};
	\node (b) [right=of a,xshift=1cm] {$b$.};
	\draw[->,bend left] (a) to [edge label={$f$}] (b);
	\draw[->,bend left] (b) to [edge label={$g$}] (a);
	\draw[->] (a) to [edge label'={$g$}][loop] (a);
	\draw[->] (b) to [edge label'={$f$}][loop] (b);
\end{tikzpicture}
\end{center} Im Prinzip haben wir einen Automaten vor uns, und das Berechnen von $\uparrow a$ und $\uparrow b$ haengt mit dem Berechnen von akzeptierten Sprachen zusammen, wobei die Funktionen $f,g$ als ``Symbole'' und die Elemente $a,b$ als ``Zustaende'' interpretiert werden. Funktioniert der Algorithmus aus \ref{§: 230823-Huet76} in dieser Algebra, oder bekommen wir wieder Probleme aufgrund von Zyklen? Auch hier kommen wir schnell in $\infty$-Rekursionen bzw. auch in widerspruechliche Rechenpfade:
\begin{align*} 
	a\Uparrow b
		&=g(f(g(a)))\Uparrow g(f(f(b)))\\
		&=g(f(g(a))\Uparrow f(f(b)))\\
		&=g(f(g(a)\Uparrow f(b))).
\end{align*} An dieser Stelle koennen wir auf zwei Arten fortfahren und beide fuehren zu unerwuenschtem Verhalten:
\begin{enumerate}
	\item Da $g(a)=a$ und $f(b)=b$, erhalten wir wieder $a\Uparrow b$ und wir koennen die Rekursion von vorne beginnen.
	\item Wir koennen einfach weiterrechnen und erhalten $g(f(x))$, da $g(a)$ und $f(b)$ unterschiedliche Funktionssymbole haben. Ist $g(f(x))$ tatsaechlich in $a\Uparrow b$? Wir haben $\downarrow g(f(x))=\{a\}$, somit ist $g(f(x))$ nicht einmal in $a\uparrow b$!
\end{enumerate}

In der obigen Algebra laesst sich $a\Uparrow b$ ganz leicht berechnen:

\begin{fact} In jeder Algebra bestehend aus nur zwei Elementen $a$ und $b$ gilt $a\Uparrow b=a\uparrow b$.
\end{fact}
\begin{proof} Jeder Term $s\in a\uparrow b$ muss $\downarrow s=\{a,b\}$ erfuellen --- daher kann es keine Generalisierung $t\in a\uparrow b$ geben  mit $t\sqsubset s$.
\end{proof}

\begin{conclusion} Die obige Definition von $dep(\mathfrak A)$ sollte stimmen und wir muessen scheinbar alle Zyklen zaehlen, auch wenn unterschiedliche Funktionen involviert sind! Das schließt aber bereits sehr viele Algebren aus --- welche interessante zyklenfreien injektiven Algebren gibt es? Mir fallen zurzeit nur die Termalgebra, $(\mathbb N,S)$, und $(\mathbb Z,S)$ ein.
\end{conclusion}

Das folgende Bsp. zeigt, warum die Algebra auch \textbf{endlich} sein muss:

\begin{example} In in der injektiven und zyklenfreien Algebra $(\mathbb Z,S)$ haben wir ebenfalls $\infty$-Rekursionen:
\begin{align*} 
	S(2)\Uparrow S(2)
		&=S(2\Uparrow 2)\\
		&=S(S(1)\Uparrow S(1))\\
		&=S(S(1\Uparrow 1))\\
		&=S(S(S(0)\Uparrow S(0)))\\
		&=S(S(S(0\Uparrow 0)))\\
		&=S(S(S(S(-1)\Uparrow S(-1))))\\
		&=\ldots.
\end{align*}
\end{example}

\begin{pseudocode}\label{pseudocode: 230823-A} Wir schreiben den obigen Algorithmus neu:
\begin{description}
	\item[Eingabe] Eine endliche zyklenfreie injektive Algebra $\mathfrak A=(A,f_1,\ldots,f_n)$, $n\geq 0$, Elemente $a,b\in A$, und eine injektive Funktion $\chi:A\times A\to X$.
	\item[Ausgabe] $a\Uparrow_\chi b$.
\end{description}
\begin{enumerate}
	\item Sei $T:=\emptyset$.
	\item Fuer jede der endliche vielen Darstellungen von $a$ und $b$ in der Form
	\begin{align*} 
		a=f_i(a_1^i,\ldots,a_{r(f_i)}^i) \quad\text{und}\quad b=f_i(b_1^i,\ldots,b_{r(f_i)}^i)
	\end{align*} fuer $1\leq i\leq n$ und $a_1^i,\ldots,a_{r(f_i)}^i,b_1^i,\ldots,b_{r(f_i)}^i\in A$, setze 
	\begin{align*} 
		T:=T\cup f_i(a_1^i\Uparrow_\chi b_1^i,\ldots,a_{r(f_i)}^i\Uparrow_\chi b_{r(f_i)}^i).
	\end{align*}
	\item Falls $T=\emptyset$, setze $T:=\{\chi(a,b)\}$.
	\item Returniere $T$.
\end{enumerate}
\end{pseudocode}

Wir versuchen nun zu zeigen, dass der Algorithmus in \ref{§: 230823-Huet76} funktioniert: 

\begin{Todo} Zeige:
\begin{description}
    \item[Terminierung] \prettyref{pseudocode: 230823-A} terminiert fuer jede Eingabe.
    \item[Soundness] \prettyref{pseudocode: 230823-A} liefert eine Menge $T$ sodass $T\subseteq a\Uparrow b$.
    \item[Completeness] \prettyref{pseudocode: 230823-A} liefert eine Menge $T$ sodass $a\Uparrow b\subseteq T$.
\end{description}
\end{Todo}

Zuerst wollen wir zeigen, dass der Algorithmus fuer jede Eingabe terminiert. Wir setzen voraus, dass es mindestens eine Darstellung der Form
\begin{align*} 
	a=f_i(a_1,\ldots,a_{r(f_i)}) \quad\text{und}\quad b=f_i(b_1,\ldots,b_{r(f_i)}),
\end{align*} fuer ein $1\leq i\leq n$ gibt. Wir haben bereits angemerkt, dass es aufgrund der Injektivitaet von $\mathfrak A$ nur endlich viele solche Darstellungen geben kann, dh die Schleife in \prettyref{pseudocode: 230823-A} wird nur endlich oft durchlaufen. Da $\mathfrak A$ zyklenfrei ist, kommt $a$ in keiner Darstellung von $a_j$, $1\leq j\leq r(f_i)$, vor, dh fuer jede Darstellung
\begin{align*} 
	a_j=f_k(c_1,\ldots,c_{r(f_k)}),\quad 1\leq k\leq n,
\end{align*} gilt $c_1,\ldots,c_{r(f_k)}\neq a$. Das Gleiche gilt fuer jedes $b_j$...

\todo[inline]{}

\section{Introduction}\label{§:I}

\todo[inline]{das Paper explizit als foundational/survey-style contribution rahmen und die Einleitung entsprechend schärfer positionieren.}

Abstraction is key to human and artificial intelligence (AI) as it allows one to see common structure in otherwise distinct objects or situations \cite{Giunchiglia92,Saitta13} and as such it is a key element for generality in computer science and artificial intelligence \cite{McCarthy87,Kramer07}. It has been studied in the fields of theorem proving (e.g. \cite{Plaisted81}) and knowledge representation and reasoning (KR\&R) by a number of authors (e.g. \cite{Knoblock94,Sacerdoti74}), and it has recently gained momentum in answer set programming \cite{Saribatur20,Saribatur21}, one of the most prominent formalisms in the field of KR\&R (see e.g. \cite{Brewka11,Lifschitz19}). \cite{Saribatur21} contain a very rich bibliography for readers interested in the literature on abstraction.

Anti-unification (or generalization) (cf. \cite{Cerna23}) is \textit{the} field of mathematical logic and theoretical computer science studying abstraction. It is the ``dual'' operation to the well-studied unification operation (cf. \cite{Baader01}). More formally, given two terms $s$ and $t$, unification searches for a substitution --- the \textit{most general unifier} --- $\sigma$ satisfying $s\sigma=t\sigma$, whereas syntactic anti-unification searches for the \textit{least general generalization} $u$ such that $s=u\sigma$ and $t=u\theta$, for some substitutions $\sigma,\theta$. Notice the difference between the two operations: while unification computes a substitution, anti-unification computes a term. % Therefore saying that the two operations are ``dual'' --- as it is often done in the literature --- is a bit misleading.

Syntactic anti-unification has been introduced by \cite{Plotkin70} and \cite{Reynolds70}, and it has found numerous applications in theoretical computer science and artificial intelligence, as for example in inductive logic programming \cite{Muggleton91} (cf. \cite{Cropper21,Cropper22}), programming by example \cite{Gulwani16}, library learning and compression \cite{Cao23}, and, in the form of $E$-generalization (i.e. anti-unification modulo theory) \cite{Heinz95,Burghardt05} in analogy-making \cite{Weller07,Schmidt14} (for further applications, see e.g. \cite{Barwell18,deSousa21,Vanhoof19}). 

The purpose of this paper is to initiate an \textit{algebraic} theory of anti-unification within general algebras. More formally, given two algebras $\mathfrak A$ and $\mathfrak B$ in the sense of general algebra (cf. \cite[§II]{Burris00}), we shall define the set of minimally general generalizations (or mggs) of two \textit{elements} of algebras (instead of terms). That is, given $a$ in $\mathfrak A$ and $b$ in $\mathfrak B$, the set $a\Uparrow_{ \mathfrak{(A,B)}} b$ of mggs will consist of all terms $s$ such that $a$ and $b$ are within the range of the term function induced by $s$ in $\mathfrak A$ and $\mathfrak B$, respectively, and $s$ is minimal with respect to a suitable algebraic generalization ordering (see \prettyref{§:AU}). Notice that this operation has \textit{no} ``dual'' in the theory of unification since it makes no sense to try to ``unify'' two elements (constant symbols) $a$ and $b$ by finding a substitution $\sigma$ such that $a\sigma=b\sigma$ (which holds iff $a=b$).

% \footnote{I am convinced that algebraic anti-unification as proposed here will find other applications in the future.}
The initial motivation for studying algebraic anti-unification as proposed in this paper are two recent applications to AI, which we shall now briefly recall:\footnote{I admit that motivating a new framework only with one owns work is a bit egocentric --- however, I do believe that the presented algebraic re-interpretation of anti-unification has an intrinsic relevance which goes beyond my own work.}% (see \prettyref{§:AI} for a more formal discussion):
\begin{description}
	% todo unpublished
	\item[\textbf{Similarity}] Detecting and exploiting similarities between seemingly distant objects is at the core of artificial general intelligence utilized for example in analogical transfer \cite{Badra18,Badra23}. The author has recently introduced an abstract algebraic notion of similarity in the general setting of general algebra, the same mathematical context underlying this paper \cite{Antic23-2}. There, the set $\uparrow_ \mathfrak A a$ of all generalizations of an element $a$ in $\mathfrak A$ naturally occurs since similarity is roughly defined as follows: two elements $a$ in $\mathfrak A$ and $b$ in $\mathfrak B$ are called \textit{\textbf{similar}} in $\mathfrak{(A,B)}$ iff either $(\uparrow_ \mathfrak A a)\cup (\uparrow_ \mathfrak B b)$ consists only of trivial generalizations generalizing all elements of $\mathfrak A$ and $\mathfrak B$; or $a\uparrow_{ \mathfrak{(A,B)}} b$ is $\subseteq$-maximal with respect to $a$ and $b$ (separately). The intuition here is that generalizations in $\uparrow_ \mathfrak A a$ encode \textit{properties} of $a$; for example, the term $2x$ is a generalization of a natural number $a\in \mathbb N$ with respect to multiplication --- that is, $2x\in\ \uparrow_{(\mathbb N,\cdot, \mathbb N)} a$ --- iff $a$ is even. While the role of the set of \textit{all} generalizations of elements is immanent from the definition of similarity, the role of the set of \textit{minimally general generalizations} as defined here is more mysterious and the content of ongoing research.

	\item[\textbf{Analogical proportions}] Analogical proportions are expressions of the form ``$a$ is to $b$ what $c$ is to $d$'' --- written $a:b::c:d$ --- at the core of analogical reasoning with numerous applications to artificial intelligence such as computational linguistics (e.g. \cite{Lepage98,Lepage01,Lepage03}), image processing (e.g. \cite{Lepage14}), recommender systems (e.g. \cite{Hug19}), and program synthesis \cite{Antic23-23}, to name a few (cf. \cite{Prade21}). The author has recently introduced an abstract algebraic framework of analogical proportions in the general setting of general algebra, the same setting as the notion of algebraic anti-unification of this paper \cite{Antic22}. It is formulated in terms of arrow proportions of the form ``$a$ transforms into $b$ as $c$ transforms into $d$'' --- written $a\to b \righttherefore c\to d$ --- and justifications of the form $s\to t$ generalizing the arrows $a\to b$ and $c\to d$. Thus, the sets of generalizations $\uparrow_ \mathfrak A(a\to b)$ and $\uparrow_ \mathfrak B(c\to d)$ naturally occur in the framework and its role is evident from the definition of proportions. However, the role of the set of \textit{minimally general generalizations} is more mysterious and the content of ongoing research --- it has already been shown in \cite[§Tree proportions]{Antic23-22}, however, that least general generalizations are key to term proportions in free term algebras which is a strong indication that algebraic anti-unification as proposed in this paper ought to play a key role for analogical proportions in the general setting (the role of anti-unification for analogical proportions has been recognized for $E$-anti-unification in other frameworks of analogical proportions by \cite{Weller07}).

	The author has recently studied (directed) logic program proportions \cite{Antic23-23} of the form $P\to Q \righttherefore R\to S$ for automated logic programming. In the process, so-called logic program forms where introduced as proper generalizations of logic programs. Given two programs $P$ and $R$, computing the set of all common forms $P\uparrow R$ and the set of all minimally general forms $P\Uparrow R$ appears challenging given the rich algebraic structure of logic programs. As in the abstract setting of analogical proportions above, the role of $\uparrow_ \mathfrak P(P\to Q)$ is evident from the definition of logic program proportions, whereas the exact role of minimally general forms is the content of ongoing research.
\end{description}

We focus here primarily on foundational issues, not on applications (but see the potential applications to similarity and analogical proportions briefly discussed above). %--- however, we do sketch two recent applications of algebraic anti-unification as proposed in this paper to the algebraic formalization of similarity and analogical proportions, both relevant for artificial intelligence (\prettyref{§:AI}).
For example, the study of structure-preserving mappings (i.e. homomorphisms) in combination with anti-unification in \prettyref{§:H} appears to be novel and characteristic to the semantic approach of this paper.

% % Werde ich wohl fuer das Einreichen bei JAIR brauchen
Algebraic anti-unification as proposed in this paper is related to \textit{$E$-anti-unification} or \textit{anti-unification modulo equational theory} \cite{Burghardt05}. In fact, if the underlying algebra has an equational axiomatization $E$, then algebraic and $E$-anti-unification are two sides of the same coin. However, the framework presented in this paper works for \textit{all} algebras, not only for those with an equational axiomatization.

\section{Generalized algebraic anti-unification}\label{§:GAAU}

In this section, we introduce the following generalization of element-wise anti-unification from above to set-wise anti-unification. In the rest of this section, $C$ is a subset of the universe $A$ of $\mathfrak A$, and $D$ is a subset of the universe $B$ of $\mathfrak B$.

\begin{definition} Define
\begin{align*} 
	\uparrow_ \mathfrak A C &:= \bigcap_{a\in C}\uparrow_ \mathfrak A a\\
	C\uparrow_{ \mathfrak{(A,B)}}D &:= (\uparrow_ \mathfrak A C)\cap (\uparrow_ \mathfrak B D)\\
	C\Uparrow_{ \mathfrak{(A,B)}}D &:= \min_{\sqsubseteq_{ \mathfrak{(A,B)}}}(C\uparrow_{ \mathfrak{(A,B)}}D)\\
	\Uparrow_ \mathfrak A C &:= C\Uparrow_ \mathfrak A C.
\end{align*}
\end{definition}

\begin{proposition} $C\uparrow_{ \mathfrak{(A,B)}}D=\bigcap_{a\in C,b\in D}[a\uparrow_{ \mathfrak{(A,B)}}b]$.
\end{proposition}
\begin{proof}
\begin{align*} 
	C\uparrow_{ \mathfrak{(A,B)}}D
		&= (\uparrow_ \mathfrak A C)\cap (\uparrow_ \mathfrak B D)\\
		&= \left[\bigcap_{a\in C}\uparrow_ \mathfrak A a\right]\cap\left[\bigcap_{b\in D}\uparrow_ \mathfrak B b\right]\\
		&= \bigcap_{a\in C}\bigcap_{b\in D}[(\uparrow_ \mathfrak A a)\cap (\uparrow_ \mathfrak B b)]\\
		&= \bigcap_{a\in C,b\in D}[a\uparrow_{ \mathfrak{(A,B)}}b].
\end{align*}
\end{proof}

\begin{fact} We have the following:
\begin{itemize}
	\item $\uparrow_ \mathfrak A \{a\}=\ \uparrow_ \mathfrak A a$.
	\item $C\subseteq\ \downarrow_ \mathfrak A s \quad\Leftrightarrow\quad  s\in\ \uparrow_ \mathfrak A C$.
	\item $C=\ \downarrow_ \mathfrak A s \quad\Rightarrow\quad  s\in\ \Uparrow_ \mathfrak A C$.
	\item $ s\in C\uparrow_ \mathfrak A D \quad\Leftrightarrow\quad C\cup D\subseteq\ \downarrow_ \mathfrak A s$.
\end{itemize} 
\end{fact}

\begin{example} We wish to compute $\Uparrow_ \mathfrak N 2\mathbb N$ in $\mathfrak N=(\mathbb N,+,\cdot,1)$. The generalization $2x$ defines exactly the even numbers in the sense that
\begin{align*} 
	\downarrow_ \mathfrak N 2x=2\mathbb N,
\end{align*} which means that
\begin{align*} 
	2x\in\ \Uparrow_ \mathfrak N 2 \mathbb N.
\end{align*} Every other generalization $s\in\ \Uparrow_ \mathfrak N 2 \mathbb N$ thus has to satisfy
\begin{align*} 
	\downarrow_ \mathfrak N s=\ \downarrow_ \mathfrak N 2x=2\mathbb N.
\end{align*}
\todo[inline]{Ist das Bsp vollstaendig?}
\end{example}

\section{Characteristic generalizations}

Anti-unifying an element with itself yields the set of minimal general generalizations of that element which, in a sense, ``characterize'' that element since generalizations encode properties of elements (for example, $2x$ is a genearlization of a natural number $a$ iff $a$ is even). This motivates the following definition:

\begin{definition} Define the set of \textit{\textbf{characteristic generalizations}} of $a\in A$ in $\mathfrak A$ by
\begin{align*} 
	\Uparrow_ \mathfrak A a:=\min_{\sqsubseteq_ \mathfrak A}\uparrow_ \mathfrak A a=a\Uparrow_ \mathfrak A a.
\end{align*}
\end{definition}

\begin{example} Recall the situation in \prettyref{e:BOOL}. We have
\begin{align*} 
	\Uparrow_ \mathfrak{BOOL} 0&=\left\{s\in T_{\{\lor,\neg,0,1\}, X} \;\middle|\; \text{$s$ is unsatisfiable}\right\},\\
	\Uparrow_ \mathfrak{BOOL} 1&=\left\{s\in T_{\{\lor,\neg,0,1\}, X} \;\middle|\; \text{$s$ is valid}\right\},\\
	0\Uparrow_ \mathfrak{BOOL} 1&=0\uparrow_ \mathfrak{BOOL} 1.
\end{align*}
\end{example}

\begin{definition} We call a set $G$ of generalizations a \textit{\textbf{characteristic set of generalizations}} of $a$ in $\mathfrak A$ iff
\begin{enumerate}
	\item $G\subseteq\ \Uparrow_ \mathfrak A a\in A$;
	\item $G\not\subseteq\ \Uparrow_ \mathfrak A b$ for all $b\neq a$.
\end{enumerate} In case $G=\{s\}$ is a singleton, we call $s$ a \textit{\textbf{characteristic generalization}} of $a$ in $\mathfrak A$.
\end{definition}

\begin{example} $X\cup X^c$ is a characteristic generalization of $U$ in $(2^U,\cup)$.
\end{example}

\begin{example} $x+(-x)$ is a characteristic generalization of $0$ in $(\mathbb Z,+)$.
\end{example}

\newpage\section*{TODOs}

\todo[inline]{Does in \prettyref{t:HT} at least $a\Uparrow_ \mathfrak A b\subseteq H(a)\Uparrow_ \mathfrak B b$ hold for [injective] homomorphisms as well?}

\todo[inline]{Examples: nullary, finitary, infinitary, trivial}

\fi
\end{document}